\documentclass{article}

\usepackage{microtype}
\usepackage{graphicx}
\usepackage{subfigure}
\usepackage{booktabs} 

\usepackage{hyperref}

\usepackage[accepted]{icml2021}

\def\calP{{\mathcal{P}}}

\def\calX{{\mathcal{X}}}

\usepackage[cmex10]{amsmath}
\usepackage{amssymb}
\usepackage{algorithm,algorithmic}

\usepackage{tikz}
\usepackage{tkz-berge}
\usepackage{lipsum}
\usetikzlibrary{calc}
\usetikzlibrary{shapes,fit}
\usepackage{verbatim}
\usepackage{setspace}
\usetikzlibrary{shapes,arrows}
\usepackage[short]{optidef}
\usepackage{pifont}

\usepackage{amsthm}

\usepackage{xcolor, colortbl}
\definecolor{LightGray}{gray}{0.9}
\definecolor{LightCyan}{rgb}{0.88,1,0.8}

\usepackage{color}
\definecolor{gray}{RGB}{128,128,128}

\usepackage{graphicx}
\graphicspath{{figure/}}
\usepackage{epstopdf}
\usepackage{tikz}
\usetikzlibrary{arrows,shapes,backgrounds}

\usepackage{bm}

\usepackage{array}
\newcolumntype{M}[1]{>{\centering\arraybackslash}m{#1}}
\newcolumntype{N}{@{}m{0pt}@{}}

\usepackage{multirow}

\usepackage{stfloats}
\usepackage{caption}

\usepackage{enumitem}
\usepackage{url}
\usepackage{adjustbox}
\usepackage{booktabs}
\usepackage{makecell}

\usepackage{footnote}

\newtheorem{theorem}{Theorem}
\newtheorem{assumption}{Assumption}

\newtheorem{lemma}{Lemma}
\newtheorem{corollary}{Corollary}
\newtheorem{definition}{Definition}

\DeclareMathOperator{\Reg}{Reg}
\DeclareMathOperator{\Dom}{Dom}

\DeclareMathOperator*{\argmin}{arg\, min}

\newenvironment{proof}[1][Proof]%
  {\smallskip\par\noindent\textbf{#1\,:\ }}%
  {\hspace*{\fill} \rule{6pt}{6pt}\smallskip}
\newenvironment{proof*}[1][Proof]%
  {\smallskip\par\noindent\textbf{#1\,:\ }}%

\newlength{\fwidth}

\allowdisplaybreaks

\icmltitlerunning{Regret and Cumulative Constraint Violation Analysis for Online Convex Optimization with Long Term Constraints}

\begin{document}

\twocolumn[
\icmltitle{Regret and Cumulative Constraint Violation Analysis\\ for Online Convex Optimization with Long Term Constraints}




\begin{icmlauthorlist}
\icmlauthor{Xinlei Yi}{kth}
\icmlauthor{Xiuxian Li}{tj}
\icmlauthor{Tao Yang}{ne}
\icmlauthor{Lihua Xie}{ntu}
\icmlauthor{Tianyou Chai}{ne}
\icmlauthor{Karl H.~Johansson}{kth}
\end{icmlauthorlist}

\icmlaffiliation{kth}{School of Electrical Engineering and Computer Science, and Digital Futures, KTH Royal Institute of Technology, Stockholm, Sweden}
\icmlaffiliation{tj}{Department of Control Science and Engineering, College of Electronics and Information Engineering, Institute for Advanced Study, and Shanghai Research Institute for Intelligent Autonomous Systems, Tongji University, Shanghai, China}
\icmlaffiliation{ne}{State Key Laboratory of Synthetical Automation for Process Industries, Northeastern University, Shenyang, China}
\icmlaffiliation{ntu}{School of Electrical and Electronic Engineering,
Nanyang Technological University, Singapore}

\icmlcorrespondingauthor{Tao Yang}{yangtao@mail.neu.edu.cn}

\icmlkeywords{Cumulative Constraint Violation, Long Term Constraints, Online Convex Optimization, Regret}

\vskip 0.3in
]



\printAffiliationsAndNotice{}  

\begin{abstract}
This paper considers online convex optimization with long term constraints, where constraints can be violated in intermediate rounds, but need to be satisfied in the long run. The cumulative constraint violation is used as the metric to measure constraint violations, which excludes the situation that strictly feasible constraints can compensate the effects of violated constraints. A novel algorithm is first proposed and it achieves an $\mathcal{O}(T^{\max\{c,1-c\}})$ bound for static regret and an $\mathcal{O}(T^{(1-c)/2})$ bound for cumulative constraint violation, where $c\in(0,1)$ is a user-defined trade-off parameter, and thus has improved performance compared with existing results. Both static regret and cumulative constraint violation bounds are reduced to $\mathcal{O}(\log(T))$ when the loss functions are strongly convex, which also improves existing results. 
In order to achieve the optimal regret with respect to any comparator sequence, another algorithm is then proposed and it achieves the optimal $\mathcal{O}(\sqrt{T(1+P_T)})$ regret and an $\mathcal{O}(\sqrt{T})$ cumulative constraint violation, where $P_T$ is the path-length of the comparator sequence. Finally, numerical simulations are provided to illustrate the effectiveness of the theoretical results.
\end{abstract}

\section{Introduction}

Online convex optimization is a promising learning framework for modeling sequential tasks and has important applications in online binary classification \cite{crammer2006online}, online display advertising \cite{goldfarb2011online}, etc.
It has been studied since the 1990's \cite{cesa1996worst,gentile1999linear,gordon1999regret,zinkevich2003online,hazan2007logarithmic,
agarwal2010optimal,shalev2012online,jadbabaie2015online,hazan2016introduction,zhang2018dynamic,pmlr-v124-zhang20a,ijcai2020-731}.
Online convex optimization can be understood as a repeated game between a learner and an adversary \cite{shalev2012online}.  At round $t$ of the game, the learner selects a point $x_t$ from a known closed convex set $\mathcal{X}\subseteq\mathbb{R}^{p}$, and the adversary chooses a convex loss function $f_t:\mathcal{X}\rightarrow \mathbb{R}$. After that, the loss function $f_t$ is revealed to the learner who suffers a loss $f_t(x_t)$. Note that at each round the loss function can be arbitrarily chosen by the adversary, especially with no probabilistic model imposed on the choices. This is the key difference between online and stochastic convex optimization.  The goal of the learner is to choose a sequence $x_{[T]}=(x_1,\dots,x_T)$ such that her regret
\begin{align}\label{online_op:reg}
\Reg(x_{[T]},y_{[T]}):=\sum_{t=1}^{T}f_t(x_{t})-\sum_{t=1}^{T}f_t(y_{t})
\end{align}
is minimized, where $T$ is the total number of rounds and $y_{[T]}=(y_1,\dots,y_T)$ is a  comparator sequence. In the literature, there are two commonly used comparator sequences. One is the optimal dynamic decision sequence $y_{[T]}=x^*_{[T]}=(x^*_{1},\dots,x^*_{T})$ solving the following constrained convex optimization problem when the sequence of loss functions is known a priori:
\begin{align*}
\min_{x_{[T]}\in \mathcal{X}^T}\sum_{t=1}^{T} f_t(x_t).
\end{align*}
In this case, $\Reg(x_{[T]},x^*_{[T]})$ is called the dynamic regret.
The other comparator sequence is the optimal static decision sequence $y_{[T]}=\check{x}^*_{[T]}=(\check{x}^*_T,\dots,\check{x}^*_T)$, where $\check{x}^*_T$ is the optimal static decision solving
\begin{align*}
\min_{x\in \mathcal{X}}\sum_{t=1}^{T} f_t(x).
\end{align*}
In this case, $\Reg(x_{[T]},\check{x}^*_{[T]})$ is called the static regret. 
In online convex optimization, we are usually interested in finding an upper bound on the worst case regret of an
algorithm. Intuitively, an algorithm performs well if its static regret is sublinear as
a function of $T$, since this implies that on the average the algorithm performs as well as the best fixed strategy in hindsight as $T$ goes to infinity \cite{shalev2012online,hazan2016introduction}.

It is known that the simple and popular projection-based online gradient descent algorithm
\begin{align}\label{online_op:intro_ogd_alg}
x_{t+1}=\calP_{\mathcal{X}}(x_t-\alpha \nabla f_t(x_t)),
\end{align}
where $\calP_{\mathcal{X}}(\cdot)$ is the projection onto the closed convex set $\mathcal{X}$ and $\alpha>0$ is the stepsize, achieves an $\mathcal{O}(\sqrt{T})$ static regret bound for loss functions with bounded subgradients \cite{zinkevich2003online}, i.e.,
\begin{align*}
\Reg(x_{[T]},\check{x}^*_{[T]})=\mathcal{O}(\sqrt{T}).
\end{align*}
It was later shown that $\mathcal{O}(\sqrt{T})$ is a tight bound up to constant factors \cite{hazan2007logarithmic}. The static regret bound can be reduced under more stringent strong convexity conditions on the loss functions \cite{hazan2007logarithmic,shalev2012online,hazan2016introduction}. When the feasible region is bounded, it was also shown in \citet{zinkevich2003online} that the algorithm \eqref{online_op:intro_ogd_alg} achieves the following regret bound
\begin{align*}
\Reg(x_{[T]},y_{[T]})=\mathcal{O}(\sqrt{T}(1+P_T)),
\end{align*}
where
\begin{align*}
P_T=\sum_{t=1}^{T-1}\|y_{t+1}-y_t\|
\end{align*}
is the path-length of the comparator sequence $y_{[T]}$. By running the projection-based online gradient descent algorithm \eqref{online_op:intro_ogd_alg} $\mathcal{O}(\log(T))$ times in parallel with different stepsizes and choosing the smallest regret through an expert-tracking algorithm, the optimal regret bound
\begin{align*}
\Reg(x_{[T]},y_{[T]})=\mathcal{O}(\sqrt{T(1+P_T)})
\end{align*}
was achieved in \citet{zhang2018adaptive}. The curvature of loss functions, such as strong convexity and smoothness, can be used to reduce the regret bound \cite{mokhtari2016online,zhang2017improved,pmlr-v97-zhang19,zhao2020dynamic,zhao2020improved}.

Despite the simplicity of the algorithm \eqref{online_op:intro_ogd_alg}, its computational cost  is crucial for its applicability. The projection $\calP_{\mathcal{X}}(\cdot)$ is easy to compute and even has a closed form solution when $\mathcal{X}$ is a simple set, e.g., a box or a ball. However, in practice, the constraint set $\mathcal{X}$ is often complex. For example, if $\mathcal{X}$ is characterized by inequalities as $\mathcal{X}=\{x:~g(x)\le\bm{0}_{m},~x\in\mathbb{X}\}$, where $\mathbb{X}\subseteq\mathbb{R}^{p}$ is a closed convex set and $g(x)=(g_1(x),\dots,g_m(x))^\top$ with each $g_i:\mathbb{R}^{p}\rightarrow \mathbb{R}$ being a convex function, then the projection $\calP_{\mathcal{X}}(\cdot)$ yields a heavy computational burden. To tackle this challenge, online convex optimization with long term constraints was considered in \citet{mahdavi2012trading}. In this case, instead of requiring $g(x_t)\le\bm{0}_{m}$ at each round, the constraint should only be satisfied in the long run. More specifically, the constraint violation
\begin{align}\label{online_op:intro_def_cons}
\Big\|\Big[\sum_{t=1}^Tg(x_{t})\Big]_+\Big\|
\end{align}
should grow sublinearly, where $[\cdot]_+$ is the projection onto the nonnegative space. In other words, the learner is allowed sometimes to make decisions that do not belong to the set $\mathcal{X}$, but the overall sequence of chosen decisions must obey the constraint at the end by a vanishing convergence rate. This problem is normally solved by online primal--dual algorithms \cite{mahdavi2012trading,jenatton2016adaptive,NIPS2018_7852,yu2020lowJMLR}. The problem can be extended to the case where the constraint functions are time-varying and revealed to the learner after her decision is chosen \cite{sun2017safety,chen2017online,neely2017online,yu2017online}. The problem can also be extended to distributed settings \cite{Li2018distributed,yi2020distributed,yi2019distributed,
yi2021regret,yuan2021distributed,yuan2021distributedtac}.

The constraint violation metric defined in \eqref{online_op:intro_def_cons} allows inequality constraint violations at many rounds as long as they are compensated by a strictly feasible constraint that has a large margin, since it takes the summation over rounds before the projection operation $[\cdot]_+$. In this way, although the constraint violation grows sublinearly, the constraints may not be satisfied at many time instances, which will restrict the theoretical results only to applications where the constraints have cumulative nature. This motivates researchers to consider stricter forms of constraint violation metric. In \citet{NIPS2018_7852}, the cumulative constraint violation
\begin{align}\label{online_op:regc1}
\Big\|\sum_{t=1}^T[g(x_{t})]_+\Big\|
\end{align}
and the cumulative squared constraint violation
\begin{align}\label{online_op:regcsq}
\sum_{t=1}^T\|[g(x_{t})]_+\|^2
\end{align}
are considered. Both forms of metric \eqref{online_op:regc1} and \eqref{online_op:regcsq} take into account all constraints that are not satisfied, thus both are stricter than the constraint violation metric defined in \eqref{online_op:intro_def_cons}. In this paper, we consider a variant of cumulative constraint violation
\begin{align}\label{online_op:regc}
\sum_{t=1}^T\|[g(x_{t})]_+\|,
\end{align}
which is equivalent to metric \eqref{online_op:regc1} due to
\begin{align*}
&\Big\|\sum_{t=1}^T[g(x_{t})]_+\Big\|\le
\sum_{t=1}^T\|[g(x_{t})]_+\|\le\sum_{t=1}^T\|[g(x_{t})]_+\|_1\\
&=\Big\|\sum_{t=1}^T[g(x_{t})]_+\Big\|_1
\le\sqrt{m}\Big\|\sum_{t=1}^T[g(x_{t})]_+\Big\|.
\end{align*}
Moreover, it is straightforward to see that the cumulative constraint violation \eqref{online_op:regc} is stricter than the cumulative squared constraint violation \eqref{online_op:regcsq} when the constraint functions are bounded.
\begin{savenotes}
\begin{table*}[htbp]
\caption{Comparison of this paper to related works on online convex optimization with long term constraints.}
\label{online_op::table}
\begin{center}
\begin{small}
\begin{tabular}{M{1.65cm}|M{2.05cm}|M{2.1cm}|M{2.15cm}|M{2.55cm}|M{4.1cm}N}
\hline
Reference&Loss functions&Static regret&Regret&Constraint violation&Cumulative constraint violation&\\[7pt]

\hline
\cite{mahdavi2012trading}&Convex&$\mathcal{O}(\sqrt{T})$& Not given& $\mathcal{O}(T^{3/4})$&Not given&\\[13pt]

\hline
\multirow{2}{*}[-4pt]{\parbox{1.5cm}{\centering \cite{jenatton2016adaptive}}}&Convex&$\mathcal{O}(T^{\max\{c,1-c\}})$&\multirow{2}{*}[-4pt]{\parbox{2.2cm}{\centering Not given}}& \multirow{2}{*}[-4pt]{\parbox{2.6cm}{\centering $\mathcal{O}(T^{1-c/2})$}}&\multirow{2}{*}[-4pt]{\parbox{4.1cm}{\centering Not given}}&\\[7pt]
\cline{2-3}
&Strongly convex&$\mathcal{O}(T^{c})$& & &&\\[7pt]

\hline
\multirow{2}{*}[-0pt]{\parbox{1.5cm}{\centering \cite{NIPS2018_7852}}}&Convex&$\mathcal{O}(T^{\max\{c,1-c\}})$&\multirow{2}{*}[-4pt]{\parbox{2.2cm}{\centering Not given}}& \multicolumn{2}{c}{$\mathcal{O}(T^{1-c/2})$}&\\[7pt]
\cline{2-3}\cline{5-6}
&Strongly convex&$\mathcal{O}(\log(T))$&& \multicolumn{2}{c}{$\mathcal{O}(\sqrt{\log(T)T})$}&\\[7pt]

\hline
\cite{yu2020lowJMLR}&Convex&$\mathcal{O}(\sqrt{T})$&Not given& $\mathcal{O}(T^{1/4})$\footnote{This bound is reduced to $\mathcal{O}(1)$ if the constraint functions satisfy the Slater condition.} &Not given&\\[13pt]

\hline
\multirow{2}{*}[-4pt]{Algorithm~\ref{online_op:algorithm2}}&Convex&$\mathcal{O}(T^{\max\{c,1-c\}})$
&$\mathcal{O}(\sqrt{T}(1+P_T))$& \multicolumn{2}{c}{$\mathcal{O}(T^{(1-c)/2})$}&\\[7pt]
\cline{2-6}
&Strongly convex&$\mathcal{O}(\log(T))$& Not given& \multicolumn{2}{c}{$\mathcal{O}(\log(T))$}&\\[7pt]

\hline
Algorithm~\ref{online_op:algorithm}&Convex&$\mathcal{O}(T^{\max\{c,1-c\}})$
& $\mathcal{O}(\sqrt{T(1+P_T)})$&\multicolumn{2}{c}{$\mathcal{O}(\sqrt{T})$}&\\[7pt]
\hline
\end{tabular}
\end{small}
\end{center}
\vskip -0.1in
\end{table*}
\end{savenotes}

\noindent {\bf  Contributions}: This paper first proposes a novel algorithm (Algorithm~\ref{online_op:algorithm2}) for the problem of online convex optimization with long term constraints. This algorithm achieves an $\mathcal{O}(T^{\max\{c,1-c\}})$ static regret bound and an $\mathcal{O}(T^{(1-c)/2})$ cumulative constraint violation bound, where $c\in(0,1)$ is a user-defined trade-off parameter, and hence yields improved performance compared with the results in \citet{mahdavi2012trading,jenatton2016adaptive,NIPS2018_7852,yu2020lowJMLR}. This algorithm is inspired by \citet{yu2020lowJMLR} and is the first to achieve a cumulative constraint violation bound strictly better than $\mathcal{O}(T^{3/4})$ while maintaining $\mathcal{O}(\sqrt{T})$ regret for convex loss functions. Both static regret and cumulative constraint violation bounds are reduced to $\mathcal{O}(\log(T))$ when the loss functions are strongly convex, which also improves the results in \citet{jenatton2016adaptive,NIPS2018_7852}. This algorithm is also the first to achieve a cumulative constraint violation bound strictly better than $\mathcal{O}(\sqrt{\log(T)T})$ while maintaining $\mathcal{O}(\log(T))$ regret for strongly convex loss functions.

In order to achieve the optimal regret with respect to any comparator sequence, another algorithm (Algorithm~\ref{online_op:algorithm}) is then proposed and it achieves the optimal $\mathcal{O}(\sqrt{T(1+P_T)})$ regret and an $\mathcal{O}(\sqrt{T})$ cumulative constraint violation. This algorithm is inspired by \citet{zhang2018adaptive}. The basic idea of the second algorithm is to run the first algorithm multiple times in parallel, each with a different stepsize that is optimal for a specific path-length, and then to combine them with an expert-tracking algorithm. This algorithm is the first to avoid computing the projection $\calP_{\mathcal{X}}(\cdot)$ by considering long term constraints while maintaining the optimal regret and sublinear cumulative constraint violation.

In summary, the presented results are significant theoretical developments compared to prior works. The comparison of this paper to related studies in the literature is summarized in Table~\ref{online_op::table}. Specifically, \citet{mahdavi2012trading} achieved an $\mathcal{O}(\sqrt{T})$ static regret bound and an $\mathcal{O}(T^{3/4})$ constraint violation bound. \citet{jenatton2016adaptive} achieved an $\mathcal{O}(T^{\max\{c,1-c\}})$ static regret bound and an $\mathcal{O}(T^{1-c/2})$ constraint violation bound, which generalized the results in \citet{mahdavi2012trading}, and the static regret bound was reduced to $\mathcal{O}(T^{c})$ when the loss functions are strongly convex. \citet{NIPS2018_7852} achieved an $\mathcal{O}(T^{\max\{c,1-c\}})$ static regret bound and an $\mathcal{O}(T^{1-c/2})$ cumulative constraint violation bound, which further generalized the results in \citet{jenatton2016adaptive} by using the stricter constraint violation metric, and these two bounds were respectively reduced to $\mathcal{O}(\log(T))$ and  $\mathcal{O}(\sqrt{\log(T)T})$ when the loss functions are strongly convex, which improved the results in \citet{jenatton2016adaptive}. \citet{yu2020lowJMLR} achieved an $\mathcal{O}(\sqrt{T})$ static regret bound and an $\mathcal{O}(T^{1/4})$ constraint violation bound, which improved the results in \citet{mahdavi2012trading,jenatton2016adaptive}.

\noindent {\bf  Outline}: The rest of this paper is organized as follows. Section~\ref{online_opsec:problem} formulates the considered problem.  Section~\ref{online_opsec:algorithm2} proposes two algorithms to solve the problem and analyze their regret and cumulative constraint violation bounds. Section~\ref{online_opsec:simulation} gives numerical simulations. Finally, Section~\ref{online_opsec:conclusion} concludes the paper and proofs are given in Appendix.

\noindent {\bf Notations}: All inequalities and equalities throughout this paper are understood componentwise. $\calX^T$ is the $T$-fold Cartesian product of a set $\calX$. $\mathbb{R}^p$ and $\mathbb{R}^p_+$ stand for the set of $p$-dimensional vectors and nonnegative vectors, respectively. $\mathbb{N}_+$ denotes the set of all positive integers. $[T]$ represents the set $\{1,\dots,T\}$ for any positive integer $T$. $\|\cdot\|$ ($\|\cdot\|_1$) represents the Euclidean norm (1-norm) for vectors and the induced 2-norm (1-norm) for matrices. $x^\top$ denotes the transpose of a vector or a matrix.   $\langle x,y\rangle$ represents the standard inner product of two vectors $x$ and $y$. ${\bf 0}_m$ denotes the column zero
vector with dimension $m$.  $[z]_+$ represents the component-wise projection of a vector $z\in\mathbb{R}^p$ onto $\mathbb{R}^p_+$. $\lceil \cdot\rceil$ and $\lfloor\cdot\rfloor$ denote the ceiling and floor functions, respectively.

\section{Problem Formulation}\label{online_opsec:problem}

\subsection{Basic Definitions}
\begin{definition}
Let $f:\Dom\rightarrow\mathbb{R}$ be a function, where the set $\Dom\subset\mathbb{R}^p$. A vector $g\in\mathbb{R}^p$ is called a subgradient of function $f$ at point $x\in\Dom$ if
\begin{align}\label{online_op:subgradient}
f(y)\ge f(x)+\langle g,y-x\rangle,~\forall y\in\Dom.
\end{align}
\end{definition}
Throughout this paper, we use $\partial f(x)$ to denote the subgradient of $f$ at $x$.
Similarly, for a vector function $\tilde{f}=(f_1,\dots,f_m)^\top:\Dom\rightarrow\mathbb{R}^m$, its subgradient at point $x\in\Dom$ is denoted as
\begin{align*}
\partial \tilde{f}(x)=\left[\begin{array}{c}(\partial f_1(x))^\top\\
(\partial f_2(x))^\top\\
\vdots\\
(\partial f_m(x))^\top
\end{array}\right]\in\mathbb{R}^{m\times p}.
\end{align*}
Moreover, it is straightforward to check that $\partial [f(x)]_+$ is the subgradient of $[f]_+$ at $x$, where
\begin{align*}
\partial [f(x)]_+=
\begin{cases}
  \bm{0}_p, & \mbox{if } f(x)<0 \\
  \partial f(x), & \mbox{otherwise}.
\end{cases}
\end{align*}

\subsection{Problem Formulation}
This paper considers the problem of online convex optimization with long term constraints.
Let $\mathbb{X}\subseteq\mathbb{R}^p$ be the constrained set and $g:\mathbb{X}\rightarrow \mathbb{R}^{m}$ be the constrained function, where $p$ and $m$ are positive integers. Both $\mathbb{X}$ and $g$ are known in advance. Suppose $\mathcal{X}=\{x:~g(x)\le\bm{0}_{m},~x\in\mathbb{X}\}$ is non-empty. Let $\{f_{t}:\mathbb{X}\rightarrow \mathbb{R}\}$ be a sequence of loss functions and each $f_t$ is unknown until the end of round $t$. The goal of this paper is to propose online algorithms to choose $x_t\in\mathbb{X}$ for each round $t$ such that both  regret and  cumulative constraint violation grow sublinearly with respect to the total number of rounds $T$.


We make the following standing assumptions on the loss and constraint functions.

\begin{assumption}\label{online_op:assfunction}
The set $\mathbb{X}$ is convex and closed. The functions $f_{t}$ and $g$ are convex.
\end{assumption}

\begin{assumption}
There exists a positive constant $F$ such that
      \begin{align}
      |f_{t}(x)-f_{t}(y)|\le F,~\|g(x)\|\le F,~\forall t\in\mathbb{N}_+,~x,~y\in \mathbb{X}.\label{online_op:ftgtupper}
      \end{align}
\end{assumption}

\begin{assumption}\label{online_op:ass_subgradient}
The subgradients $\partial f_{t}(x)$ and $\partial g(x)$ exist. Moreover, they are uniformly bounded on $\mathbb{X}$, i.e., there exists a positive constant $G$ such that
  \begin{align}\label{online_op:subgupper}
  \|\partial f_{t}(x)\|\le G,~
  \|\partial g(x)\|\le G,~\forall t\in\mathbb{N}_+,~x\in \mathbb{X}.
  \end{align}
\end{assumption}

From \eqref{online_op:ftgtupper}, we know that the cumulative constraint violation \eqref{online_op:regc} is stricter than the cumulative squared constraint violation \eqref{online_op:regcsq} due to
\begin{align*}
\sum_{t=1}^T\|[g(x_{t})]_+\|^2\le F\sum_{t=1}^T\|[g(x_{t})]_+\|.
\end{align*}

\section{Main Results}\label{online_opsec:algorithm2}
In this section, we propose two novel algorithms for the constrained online convex optimization problem  formulated in Section \ref{online_opsec:problem},  and analyze their regret and cumulative constraint violation bounds.

\subsection{The Basic Approach}

The basic approach is summarized in Algorithm~\ref{online_op:algorithm2}, which is inspired by Algorithm~1 proposed in \citet{yu2020lowJMLR}. The key difference between Algorithm~\ref{online_op:algorithm2} and the algorithm proposed in \citet{yu2020lowJMLR} is that we use the clipped constraint function $[g]_+$ to replace the original constraint function $g$. With this modification, we can analyze constraint violations under  stricter forms of metric. In addition to using the clipped constraint function, Algorithm~\ref{online_op:algorithm2} also has time-varying algorithm parameters. This enables us to consider the nontrivial extensions, such as strongly convex loss functions and the general dynamic regret, which have not been studied in \citet{yu2020lowJMLR}.

\begin{algorithm}
\caption{}
\begin{algorithmic}\label{online_op:algorithm2}
\STATE \textbf{Input}:   non-increasing sequence $\{\alpha_t\}\subseteq(0,+\infty)$ and non-decreasing sequence $\{\gamma_t\}\subseteq(0,+\infty)$.
\STATE \textbf{Initialize}:  $q_{0}={\bm 0}_{m}$ and $x_{1}\in\mathbb{X}$.
\FOR{$t=2,\dots$}
\STATE  Observe $\partial f_{t-1}(x_{t-1})$.
\STATE  Update \begin{align}
q_{t-1}&=q_{t-2}+\gamma_{t-1}[g(x_{t-1})]_+,\label{online_op:al_q2}\\
\hat{q}_{t-1}&=q_{t-1}+\gamma_{t-1}[g(x_{t-1})]_+,\label{online_op:al_qhat2}\\
       x_{t}&=\argmin_{x\in\mathbb{X}}\{\alpha_{t-1}\langle\partial f_{t-1}(x_{t-1}), x\rangle\nonumber\\
       &\quad+\alpha_{t-1}\gamma_t\langle\hat{q}_{t-1}, [g(x)]_+\rangle+\|x-x_{t-1}\|^2\}.\label{online_op:al_x2}
       \end{align}
\ENDFOR
\STATE  \textbf{Output}: $\{x_{t}\}$.
\end{algorithmic}
\end{algorithm}

In the following, we analyze regret and cumulative constraint violation bounds for Algorithm~\ref{online_op:algorithm2}.
We first provide regret and cumulative constraint violation bounds for the general cases in the following lemma.
\begin{lemma}\label{online_op:theoremreg_alg2}
Suppose Assumptions~\ref{online_op:assfunction}--\ref{online_op:ass_subgradient} hold. Let $\{x_{t}\}$ be the sequence generated by Algorithm~\ref{online_op:algorithm2} with $\gamma_t=\gamma_0/\sqrt{\alpha_t}$, where $\gamma_0\in(0,1/(\sqrt{2}G)]$ is a constant. Then, for any comparator sequence $y_{[T]}\in\calX^{T}$,
\begin{align}
&\Reg(x_{[T]},y_{[T]})\le  \sum_{t=1}^T\Delta_{t}(y_t)
+\sum_{t=1}^{T}\frac{G^2\alpha_{t}}{2},\label{online_op:theoremregequ_alg2}\\
&\sum_{t=1}^{T} \|[g(x_{t})]_+\|\nonumber\\
&\le\sqrt{m}\Big(\frac{1}{\gamma_T}\|q_{T}\|
+\sum_{t=1}^{T-1}\Big(\frac{1}{\gamma_{t}}-\frac{1}{\gamma_{t+1}}\Big)\|q_{t}\|\Big),
\label{online_op:theoremconsequ_alg2}\\
&\frac{1}{2}\|q_{T+1}\|^2\nonumber\\
&\le  \sum_{t=1}^T\Delta_{t}(y_t)
+\sum_{t=1}^{T}\frac{G^2\alpha_{t}}{2}
+\frac{1}{2}\gamma_1^2F^2+FT\nonumber\\
&\quad
+\sum_{t=1}^T(\gamma_{t+1}-\gamma_{t})^2F^2,\label{online_op:theoremregequ_alg2_g}
\end{align}
where
\begin{align*}
&\Delta_{t}(y_t)=\frac{1}{\alpha_{t}}(\|y_t-x_{t}\|^2-\|y_t-x_{t+1}\|^2).
\end{align*}
\end{lemma}
\begin{proof}
The proof is given in  Appendix~\ref{online_op:theoremreg_alg2proof}.
\end{proof}

We then present the first main result.
\begin{theorem}\label{online_op:corollaryreg}
Suppose Assumptions~\ref{online_op:assfunction}--\ref{online_op:ass_subgradient} hold. For any $T\in\mathbb{N}_+$, let $x_{[T]}$ be the sequence generated by Algorithm~\ref{online_op:algorithm2} with
\begin{align}\label{online_op:stepsize1}
  \alpha_t=\frac{\alpha_0}{T^{c}},
  ~\gamma_t=\frac{\gamma_0}{\sqrt{\alpha_t}},~\forall t\in[T],
\end{align} where $\alpha_0>0$, $c\in(0,1)$, and $\gamma_0\in(0,1/(\sqrt{2}G)]$ are constants. Then,
\begin{align}
&\Reg(x_{[T]},\check{x}^*_{[T]})
=\mathcal{O}(T^{\max\{c,1-c\}}),\label{online_op:corollaryregequ1}\\
&\sum_{t=1}^T\|[g(x_{t})]_+\|=
\mathcal{O}(T^{(1-c)/2}).
\label{online_op:corollaryconsequ}
\end{align}
\end{theorem}
\begin{proof}
The explicit expressions of the right-hand sides of \eqref{online_op:corollaryregequ1}--\eqref{online_op:corollaryconsequ}, and the proof are given in  Appendix~\ref{online_op:corollaryregproof}.
\end{proof}

Compared with the results that $\Reg(x_{[T]},\check{x}^*_{[T]})
=\mathcal{O}(T^{\max\{c,1-c\}})$  and $\|[\sum_{t=1}^Tg(x_{t})]_+\|=\mathcal{O}(T^{1-c/2})$ achieved in \citet{jenatton2016adaptive},
from \eqref{online_op:corollaryregequ1} and \eqref{online_op:corollaryconsequ}, we know that Algorithm~\ref{online_op:algorithm2} achieves the same static regret bound but a strictly smaller constraint violation bound under the stricter metric. Similarly, compared with the results that $\Reg(x_{[T]},\check{x}^*_{[T]})
=\mathcal{O}(T^{\max\{c,1-c\}})$  and $\|\sum_{t=1}^T[g(x_{t})]_+\|=\mathcal{O}(T^{1-c/2})$ achieved in \citet{NIPS2018_7852},
from \eqref{online_op:corollaryregequ1} and \eqref{online_op:corollaryconsequ}, we know that Algorithm~\ref{online_op:algorithm2} achieves the same static regret bound but a strictly smaller cumulative constraint violation bound.
By setting $c=0.5$ in Theorem~\ref{online_op:corollaryreg}, we have $\Reg(x_{[T]},\check{x}^*_{[T]})=\mathcal{O}(\sqrt{T})$ and $\sum_{t=1}^T\|[g(x_{t})]_+\|=
\mathcal{O}(T^{1/4})$. Thus, the optimal $\mathcal{O}(\sqrt{T})$ regret bound achieved in \citet{zinkevich2003online,mahdavi2012trading,yu2020lowJMLR} is recovered. Moreover, the $\mathcal{O}(T^{1/4})$ constraint violation bound achieved in \citet{yu2020lowJMLR} is improved by using the stricter metric and the $\mathcal{O}(T^{3/4})$ constraint violation bound achieved in \citet{mahdavi2012trading} is not only improved by using the stricter metric but also strictly reduced. However, the authors of \citet{yu2020lowJMLR} also showed that constraint violation bound is reduced to $\mathcal{O}(1)$ if the constraint functions satisfy the Slater condition.
Noting this, it is natural to ask the question whether Algorithm~\ref{online_op:algorithm2} can achieve $\mathcal{O}(1)$ cumulative constraint violation bound if the Slater condition holds. Unfortunately, we have not found a way to show this. The reason is that the clipped constraint functions do not satisfy the Slater condition since they are always nonnegative.

If the loss functions $f_{t}$ are strongly convex, then the static regret and cumulative constraint violation bounds achieved in Theorem~\ref{online_op:corollaryreg} can be reduced. Moreover, the total number of rounds $T$ is not needed.

\begin{corollary}\label{online_op:corollaryreg_sc}
Suppose Assumptions~\ref{online_op:assfunction}--\ref{online_op:ass_subgradient} hold. Moreover, for all $t\in\mathbb{N}_+$, $f_{t}(x)$ are strongly convex over $\mathbb{X}$,  i.e., there exists a constant $\mu>0$, such that for all $x,y\in\mathbb{X}$,
\begin{align}\label{online_op:assstrongconvexequ}
f_{t}(x)\ge f_{t}(y)+\langle x-y,\partial f_{t}(y)\rangle+\mu\|x-y\|^2.
\end{align}
Let $\{x_{t}\}$ be the sequence generated by Algorithm~\ref{online_op:algorithm2} with
\begin{align}\label{online_op:stepsize2}
  \alpha_t=\frac{1}{t\mu},
  ~\gamma_t=\frac{\gamma_0}{\sqrt{\alpha_t}},~\forall t\in\mathbb{N}_+,
\end{align} where $\gamma_0\in(0,1/(\sqrt{2}G)]$ is a constant. Then,
\begin{align}
&\Reg(x_{[T]},\check{x}^*_{[T]})=\mathcal{O}(\log(T)),\label{online_op:corollaryregequ1_sc}\\
&\sum_{t=1}^T\|[g(x_{t})]_+\|=\mathcal{O}(\log(T)).\label{online_op:corollaryconsequ_sc}
\end{align}
\end{corollary}
\begin{proof}
The explicit expressions of the right-hand sides of \eqref{online_op:corollaryregequ1_sc}--\eqref{online_op:corollaryconsequ_sc}, and the proof are given in Appendix~\ref{online_op:corollaryregproof_sc}.
\end{proof}

Corollary~\ref{online_op:corollaryreg_sc} is the first to provide the $\mathcal{O}(\log(T))$ regret and cumulative constraint violation for online convex optimization with long term constraints when the loss functions are strongly convex. While $\mathcal{O}(\log(T))$ regret is well known for traditional online convex optimization without long term constraints \cite{hazan2007logarithmic}, whether similar bounds exist for online convex optimization with long term constraints is an open problem. Thus, Corollary~\ref{online_op:corollaryreg_sc} is a significant result which requires non-trivial analysis.
Compared with the results that $\Reg(x_{[T]},\check{x}^*_{[T]})
=\mathcal{O}(T^{c})$  and $\|[\sum_{t=1}^Tg(x_{t})]_+\|=\mathcal{O}(T^{1-c/2})$ achieved in \citet{jenatton2016adaptive},
from \eqref{online_op:corollaryregequ1_sc} and \eqref{online_op:corollaryconsequ_sc}, we know that Algorithm~\ref{online_op:algorithm2} achieves strictly smaller static regret and constraint violation bounds under the stricter form of constraint violation metric. Similarly, compared with the results that $\Reg(x_{[T]},\check{x}^*_{[T]})
=\mathcal{O}(\log(T))$  and $\|\sum_{t=1}^T[g(x_{t})]_+\|=\mathcal{O}(\sqrt{\log(T)T})$ achieved in \citet{NIPS2018_7852},
from \eqref{online_op:corollaryregequ1_sc} and \eqref{online_op:corollaryconsequ_sc}, we know that Algorithm~\ref{online_op:algorithm2} achieves the same static regret bound but a strictly smaller cumulative constraint violation bound.

To end this section, let us analyze the bound of general regret for Algorithm~\ref{online_op:algorithm2}.
If the  set $\mathbb{X}$  has bounded diameter, then we can show that  Algorithm~\ref{online_op:algorithm2} achieves $\mathcal{O}(\sqrt{T}(1+P_T))$  regret. Moreover, the total number of rounds $T$ is not needed.
\begin{assumption}\label{online_op:ass_set_bounded}
The set $\mathbb{X}$  has bounded diameter, i.e., there is a positive constant $d(\mathbb{X})$ such that
\begin{align}
\|x-y\|\le d(\mathbb{X}),~\forall x,y\in\mathbb{X}.\label{online_op:domainupper}
\end{align}
\end{assumption}
Assumption~\ref{online_op:ass_set_bounded} is commonly used in the literature, e.g., \cite{zinkevich2003online,mahdavi2012trading,jenatton2016adaptive,NIPS2018_7852,zhang2018adaptive}.

\begin{theorem}\label{online_op:corollaryreg_dr}
Suppose Assumptions~\ref{online_op:assfunction}--\ref{online_op:ass_set_bounded} hold. Let $\{x_{t}\}$ be the sequence generated by Algorithm~\ref{online_op:algorithm2} with
\begin{align}\label{online_op:stepsize1_dr}
  \alpha_t=\frac{\alpha_0}{t^{c}},
  ~\gamma_t=\frac{\gamma_0}{\sqrt{\alpha_t}},~\forall t\in\mathbb{N}_+,
\end{align} where $\alpha_0>0$, $c\in(0,1)$, and $\gamma_0\in(0,1/(\sqrt{2}G)]$ are constants. Then,
\begin{align}
&\Reg(x_{[T]},y_{[T]})
=\mathcal{O}(T^{1-c}+T^c(1+P_T)),\label{online_op:corollaryregequ1_dr}\\
&\sum_{t=1}^T\|[g(x_{t})]_+\|=
\mathcal{O}(T^{(1-c)/2}).
\label{online_op:corollaryconsequ_dr}
\end{align}
\end{theorem}
\begin{proof}
The explicit expressions of the right-hand sides of \eqref{online_op:corollaryregequ1_dr}--\eqref{online_op:corollaryconsequ_dr}, and the proof are given in  Appendix~\ref{online_op:corollaryregproof_dr}.
\end{proof}

If the optimal static decision sequence is chosen as the comparator sequence, i.e., $y_{[T]}=\check{x}^*_{[T]}$, then $P_T=0$. In this case, the results in Theorem~\ref{online_op:corollaryreg_dr} recover the results in Theorem~\ref{online_op:corollaryreg}.
By setting $c=0.5$ in Theorem~\ref{online_op:corollaryreg_dr}, we have $\Reg(x_{[T]},y_{[T]})=\mathcal{O}(\sqrt{T}(1+P_T))$ and $\sum_{t=1}^T\|[g(x_{t})]_+\|=
\mathcal{O}(T^{1/4})$, which recover the regret bound achieved in \citet{zinkevich2003online}. If the path-length $P_T$ is known in advance, then the regret bound can be reduced under the cost that the cumulative constraint violation bound is increased.
\begin{corollary}\label{online_op:corollaryreg_dr2}
For any $T\in\mathbb{N}_+$, suppose the path-length $P_T$ is known in advance.
Under the same conditions stated in Theorem~\ref{online_op:corollaryreg_dr} with $\alpha_0=(1+P_T)^c$,
\begin{align}
&\Reg(x_{[T]},y_{[T]})
=\mathcal{O}(T^c(1+P_T)^{1-c}+T^{1-c}(1+P_T)^c),\label{online_op:corollaryregequ1_dr2}\\
&\sum_{t=1}^T\|[g(x_{t})]_+\|=
\mathcal{O}(T^{(1-c)/2}(1+P_T)^{c/2}).
\label{online_op:corollaryconsequ_dr2}
\end{align}
\end{corollary}
\begin{proof}
The proof follows that of Theorem~\ref{online_op:corollaryreg_dr} given in  Appendix~\ref{online_op:corollaryregproof_dr} and the fact that $\alpha_0=(1+P_T)^c\le(1+d(\mathbb{X}))^cT^c$ under Assumption~\ref{online_op:ass_set_bounded}.
The explicit expressions of the right-hand sides of \eqref{online_op:corollaryregequ1_dr2}--\eqref{online_op:corollaryconsequ_dr2} are also given in  Appendix~\ref{online_op:corollaryregproof_dr}.
\end{proof}

By setting $c=0.5$ in Corollary~\ref{online_op:corollaryreg_dr2}, we have $\Reg(x_{[T]},y_{[T]})=\mathcal{O}(\sqrt{T(1+P_T)})$ and $\sum_{t=1}^T\|[g(x_{t})]_+\|=
\mathcal{O}(T^{1/4}(1+P_T)^{1/4})$, which recover the optimal regret bound achieved in \citet{zhang2018adaptive}.
However, in Corollary~\ref{online_op:corollaryreg_dr2} the path-length $P_T$ needs to be known in advance, which is normally unknown in practice. Thus, Algorithm~\ref{online_op:algorithm2} cannot achieve the optimal $\mathcal{O}(\sqrt{T(1+P_T)})$  regret bound in general. This motivates us to propose another algorithm such that it can achieve the optimal regret bound without using $P_T$, which is presented in the next section.

\subsection{An Improved Approach}
In this section, we propose another algorithm for the constrained online convex optimization problem formulated in Section \ref{online_opsec:problem}, which achieves $\mathcal{O}(\sqrt{T(1+P_T)})$  regret and $\mathcal{O}(\sqrt{T})$ cumulative constraint violation without using the path-length $P_T$ to design the algorithm parameters.

The proposed algorithm is summarized in Algorithm~\ref{online_op:algorithm}, which is inspired by the expert-tracking algorithm (Improved Ader) proposed in \citet{zhang2018adaptive} as well as Algorithm~\ref{online_op:algorithm2}. The basic idea of Algorithm~\ref{online_op:algorithm} is to run Algorithm~\ref{online_op:algorithm2} multiple times in parallel, each with a different stepsize that is optimal for a specific path-length, and then to combine them with an expert-tracking algorithm. By setting $N=1$ in Algorithm~\ref{online_op:algorithm}, we know that Algorithm~\ref{online_op:algorithm} becomes Algorithm~\ref{online_op:algorithm2}.
Algorithm~\ref{online_op:algorithm} is designed for online convex optimization with long term constraints, while the Improved Ader algorithm in \citet{zhang2018adaptive} is for traditional online convex optimization without long term constraints.
In other words, the main difference between Algorithm~\ref{online_op:algorithm} and the Improved Ader algorithm is that we avoid computing the projection $\calP_{\mathcal{X}}(\cdot)$ by considering long term constraints. Just as the algorithms in \citet{mahdavi2012trading,jenatton2016adaptive,NIPS2018_7852,yu2020lowJMLR} are the nontrivial extensions of the classic online gradient descent algorithm in \citet{zinkevich2003online}, Algorithm~\ref{online_op:algorithm} is a nontrivial extension of the Improved Ader algorithm.

\begin{algorithm}
\caption{}
\begin{algorithmic}\label{online_op:algorithm}
\STATE \textbf{Input}: parameters $N\in\mathbb{N}_+$ and $\beta>0$;  non-increasing sequence $\{\alpha_{i,t}\}\subseteq(0,+\infty)$ and non-decreasing sequence $\{\gamma_{i,t}\}\subseteq(0,+\infty),~\forall i\in[N]$.
\STATE \textbf{Initialize}:  $q_{i,0}={\bm 0}_{m}$, $x_{i,1}\in\mathbb{X}$, $w_{i,1}=\frac{N+1}{i(i+1)N},~\forall i\in[N]$, and $x_1=\sum_{i=1}^{N}w_{i,1}x_{i,1}$.
\FOR{$t=2,\dots$}
\STATE  Observe $\partial f_{t-1}(x_{t-1})$.
\STATE  Update \begin{align}
q_{i,t-1}&=q_{i,t-2}+\gamma_{i,t-1}[g(x_{i,t-1})]_+,\label{online_op:al_q}\\
\hat{q}_{i,t-1}&=q_{i,t-1}+\gamma_{i,t-1}[g(x_{i,t-1})]_+,\label{online_op:al_qhat}\\
x_{i,t}&=\argmin_{x\in\mathbb{X}}\{\alpha_{i,t-1}\langle\partial f_{t-1}(x_{t-1}), x\rangle\nonumber\\
&\quad+\alpha_{i,t-1}\gamma_{i,t}\langle\hat{q}_{i,t-1}, [g(x)]_+\rangle\nonumber\\
&\quad+\|x-x_{i,t-1}\|^2\},\label{online_op:al_xi}\\
l_{i,t-1}&=\langle\partial f_{t-1}(x_{t-1}),x_{i,t-1}-x_{t-1}\rangle,\label{online_op:al_l}\\
w_{i,t}&=\frac{w_{i,t-1}e^{-\beta l_{i,t-1}}}{\sum_{i=1}^{N}w_{i,t-1}e^{-\beta l_{i,t-1}}},\label{online_op:al_w}\\
x_{t}&=\sum_{i=1}^{N}w_{i,t}x_{i,t}.\label{online_op:al_x}
\end{align}
\ENDFOR
\STATE  \textbf{Output}: $\{x_{t}\}$.
\end{algorithmic}
\end{algorithm}

In Algorithm~\ref{online_op:algorithm}, for each $i\in[N]$, the equations \eqref{online_op:al_q}--\eqref{online_op:al_l} can be updated in parallel. Define the surrogate loss \cite{NIPS2016_14cfdb59,zhang2018adaptive} as
\begin{align}\label{online_op:ell}
\ell_t(x)=\langle \partial f_t(x_t),x-x_t\rangle,~\forall x\in\mathbb{X}.
\end{align}
It is straightforward to see that $\partial\ell_t(x_{i,t})=\partial f_t(x_t)$. Thus, for each $i\in[N]$, the updating equations \eqref{online_op:al_q}--\eqref{online_op:al_xi} are exactly \eqref{online_op:al_q2}--\eqref{online_op:al_x2} for solving constrained online convex optimization with loss functions $\{\ell_t(x)\}$.

The regret and cumulative constraint violation bounds for Algorithm~\ref{online_op:algorithm} are provided in the following theorem.
\begin{theorem}\label{online_op:corollaryreg_alg1}
Suppose Assumptions~\ref{online_op:assfunction}--\ref{online_op:ass_set_bounded} hold. For any $T\in\mathbb{N}_+$, let $x_{[T]}$ be the sequence generated by Algorithm~\ref{online_op:algorithm} with
\begin{align}\label{online_op:stepsize3}
&N=\lceil \kappa\log_2(1+T)\rceil+1,~\beta=\frac{\beta_0}{T^c},\nonumber\\
&\alpha_{i,t}=\frac{\alpha_02^{i-1}}{t^c},
  ~\gamma_{i,t}=\frac{\gamma_0}{\sqrt{\alpha_{i,t}}},~\forall t\in[T],
\end{align} where $\kappa\in[0,c]$, $c\in(0,1)$, $\alpha_0>0$, $\beta_0>0$, and $\gamma_0\in(0,1/(\sqrt{2}G)]$ are constants. Then,
\begin{align}
&\Reg(x_{[T]},y_{[T]})
=\mathcal{O}(T^c(1+P_T)^{1-\kappa}+T^{1-c}(1+P_T)^\kappa),\label{online_op:corollaryregequ1_alg1}\\
&\sum_{t=1}^T\|[g(x_{t})]_+\|=
\mathcal{O}(T^{(1-c+\kappa)/2}).
\label{online_op:corollaryconsequ_alg1}
\end{align}
\end{theorem}
\begin{proof}
The explicit expressions of the right-hand sides of \eqref{online_op:corollaryregequ1_alg1}--\eqref{online_op:corollaryconsequ_alg1}, and the proof are given in  Appendix~\ref{online_op:corollaryregproof_alg1}.
\end{proof}

The bounds presented in \eqref{online_op:corollaryregequ1_alg1}--\eqref{online_op:corollaryconsequ_alg1} still hold if choosing $\alpha_{i,t}=\alpha_02^{i-1}/T^c$ in \eqref{online_op:stepsize3}.
By setting $\kappa=c=0.5$ in Theorem~\ref{online_op:corollaryreg_alg1}, we have $\Reg(x_{[T]},y_{[T]})=\mathcal{O}(\sqrt{T(1+P_T)})$ and $\sum_{t=1}^T\|[g(x_{t})]_+\|=
\mathcal{O}(\sqrt{T})$, which recover the optimal regret bound achieved in \citet{zhang2018adaptive} where long term constraints are not considered. Therefore, Theorem~\ref{online_op:corollaryreg_alg1} is the first to achieve the optimal regret bound for online convex optimization with long term constraints without knowing the path-length $P_T$ while maintaining sublinear cumulative constraint violation. Due to the long term constraints, the analysis is much more complicated. On the other hand, by comparing Corollary~\ref{online_op:corollaryreg_dr2} and Theorem~\ref{online_op:corollaryreg_alg1}, we see that Algorithm~\ref{online_op:algorithm} can achieve the optimal regret bound without knowing $P_T$ under the cost that the (theoretical) cumulative constraint violation bound is increased.

\section{Simulations}\label{online_opsec:simulation}
In this section, we illustrate and verify the proposed algorithms through numerical simulations.

\subsection{Online Linear Programming}
Similar to \citet{yu2020lowJMLR}, we consider online convex optimization with linear loss functions $f_t(x)=\langle\theta_t,x\rangle$, where $\theta_t\in\mathbb{R}^p$ is time-varying and unknown at round $t$; constraint set $\mathbb{X}\subseteq\mathbb{R}^p$; and constraint functions $Ax\le b$, where $A\in\mathbb{R}^{m\times p}$ and $b\in\mathbb{R}^m$ are fixed and known in advance. This problem has broad applications in various areas including transportation, energy, telecommunications, and manufacturing.

In the simulations, similar to \citet{yu2020lowJMLR}, we choose $p=2$, $m=3$, and $\mathbb{X}=[-1,1]^2$. Components of $A$ and $b$ are uniformly distributed random numbers in the intervals $[0,2]$ and $[0,5]$, respectively. The total number of rounds is chosen as $T=5000$. The time-varying loss coefficients are set as $\theta_t=\theta_{1,t}+\theta_{2,t}+\theta_{3,t}$, where components of $\theta_{1,t}$ are uniformly distributed random numbers in the interval $[-t^{1/10},t^{1/10}]$; components of $\theta_{2,t}$ are uniformly distributed random numbers in the interval $[-1,0]$ when $t\in[1,1500]\cup[2000,3500]\cup[4000,5000]$ and in the interval $[0,1]$ otherwise; and $\theta_{3,t}=((-1)^{\mu(t)},(-1)^{\mu(t)})^\top$, where $\mu(t)$ is a random permutation of the integers from 1 to $T$.

We compare Algorithms~\ref{online_op:algorithm2} and \ref{online_op:algorithm} with state-of-the-art algorithms: Algorithm~1 in \citet{yu2020lowJMLR},  Algorithm~1 in \citet{NIPS2018_7852}, the algorithm in \citet{jenatton2016adaptive}, and Algorithm~1 in \citet{mahdavi2012trading}. Table~\ref{tab:para} lists all the algorithm parameters used in the simulations\footnote{ If we write Algorithm~\ref{online_op:algorithm2} and the algorithm in \citet{yu2020lowJMLR} in the same form, then we can see that they use the same algorithm parameters.}. Figures~\ref{online_op:fig:loss} and \ref{online_op:fig:con} illustrate the evolutions of the cumulative loss $\sum_{t=1}^{T}f_t(x_t)$ and the cumulative constraint violation $\sum_{t=1}^T\|[g(x_{t})]_+\|$ averaged over 1000 independent experiments generated from the above settings, respectively. Figure~\ref{online_op:fig:loss} shows that Algorithm~\ref{online_op:algorithm} has the smallest cumulative loss and  Algorithm~\ref{online_op:algorithm2} has smaller cumulative loss than other algorithms. Figure~\ref{online_op:fig:con} shows that  Algorithms~\ref{online_op:algorithm2} and \ref{online_op:algorithm} have almost the same  cumulative constraint violation which is slightly smaller than that achieved by the algorithm in \citet{yu2020lowJMLR}, and smaller than that achieved by the algorithm in \citet{NIPS2018_7852}, and much smaller than that achieved by the algorithms in \citet{jenatton2016adaptive,mahdavi2012trading}. Therefore, the simulation results are in accordance with the theoretical results summarized in Table~\ref{online_op::table}.

\begin{table*}[ht!]
\caption{Parameters for each algorithm in online linear programming.}
\label{tab:para}
\begin{center}
\begin{small}
\begin{tabular}{M{3.5cm}|M{8.0cm}N}
\hline
Algorithm  &Parameters&\\[7pt]

\hline
Algorithm~\ref{online_op:algorithm2}        & $\alpha_t=2/T^c$, $\gamma_t = T^{c/2}$, $c=0.5$ &\\[7pt]

\hline
Algorithm~\ref{online_op:algorithm}         & $\alpha_{i,t}=2/T^c$, $\beta=3/T^c$, $\gamma_{i,t} = T^{c/2}$, $c=\kappa=0.5$ & \\[7pt]

\hline
\citet{yu2020lowJMLR}       & $\alpha=\sqrt{T}/4$, $\gamma=T^{1/4}$ &\\[7pt]

\hline
\citet{NIPS2018_7852}         & $\sigma=0.25$, $\eta=1.5/T^c$, $c=0.5$ &\\[7pt]

\hline
\citet{jenatton2016adaptive}         & $\theta_t=0.01/t^c$, $\eta_t=0.7/t^c$, $\mu_t=1/(t+1)/\theta_t$, $c=0.5$ &\\[7pt]

\hline
\citet{mahdavi2012trading}      & $\delta=0.5$, $\eta=0.8/\sqrt{T}$ &\\[7pt]

\hline
\end{tabular}
\end{small}
\end{center}
\vskip -0.1in
\end{table*}

\begin{figure}[!ht]
\centering
  \includegraphics[width=0.5\textwidth]{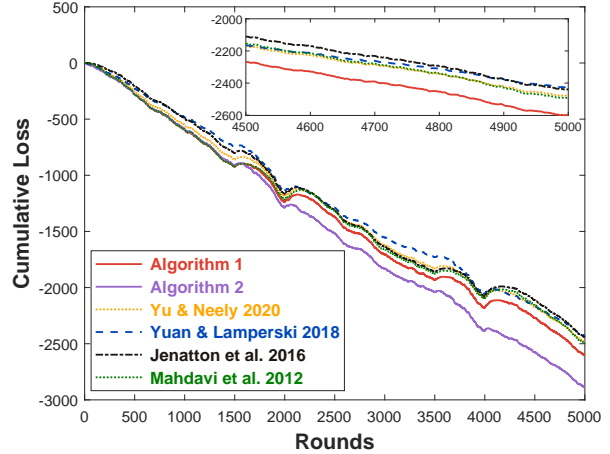}
  \caption{Evolutions of the cumulative loss in online linear programming.}
  \label{online_op:fig:loss}
\end{figure}

\begin{figure}[!ht]
\centering
  \includegraphics[width=0.5\textwidth]{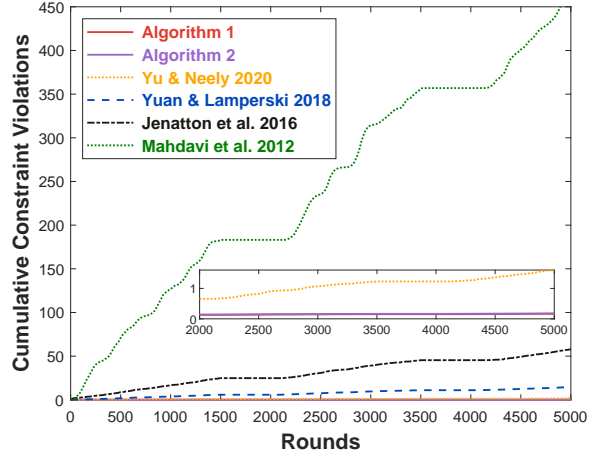}
  \caption{Evolutions of the cumulative constraint violation in online linear programming.}
  \label{online_op:fig:con}
\end{figure}

\subsection{Online Quadratic Programming}
In order to verify the improved theoretical results for strongly convex functions, we replace the linear loss functions in the above simulations by quadratic loss functions. Specifically, we consider $f_t(x)=\|x-\theta_t\|^2+20\langle\theta_t,x\rangle$.

We compare Algorithm~\ref{online_op:algorithm2} with the algorithms in \citet{NIPS2018_7852,jenatton2016adaptive} since the performance of these algorithms under the strong convexity assumption has been analyzed. Table~\ref{tab:paraSC} lists all the algorithm parameters used in the simulations. Figures~\ref{online_op:fig:lossSC} and \ref{online_op:fig:conSC} illustrate the evolutions of the cumulative loss $\sum_{t=1}^{T}f_t(x_t)$ and the cumulative constraint violation $\sum_{t=1}^T\|[g(x_{t})]_+\|$ averaged over 1000 independent experiments, respectively. Figure~\ref{online_op:fig:lossSC} shows that Algorithm~\ref{online_op:algorithm2} and the algorithm in \citet{NIPS2018_7852} have almost the same cumulative loss, which is smaller than that achieved by the algorithm in \citet{jenatton2016adaptive}. Figure~\ref{online_op:fig:conSC} shows that Algorithm~\ref{online_op:algorithm2} has significant smaller cumulative constraint violation than that achieved by the algorithms in \citet{NIPS2018_7852,jenatton2016adaptive}. Therefore, the simulation results match the theoretical results summarized in Table~\ref{online_op::table}.

\begin{table*}[ht!]
\caption{Parameters for each algorithm in online quadratic programming.}
\label{tab:paraSC}
\begin{center}
\begin{small}
\begin{tabular}{M{3.5cm}|M{8.0cm}N}
\hline
Algorithm  &Parameters&\\[7pt]

\hline
Algorithm~\ref{online_op:algorithm2}        & $\alpha_t=6/t$, $\gamma_t = \sqrt{t}$ &\\[7pt]

\hline
\citet{NIPS2018_7852}         & $\theta_t=4.5/(t+1)$, $\eta_t=3/(t+1)$ &\\[7pt]

\hline
\citet{jenatton2016adaptive}         & $\theta_t=c/t^{0.01}$, $\eta_t=2.5/t$, $\mu_t=1/(t+1)/\theta_t$, $c=0.001$ &\\[7pt]

\hline
\end{tabular}
\end{small}
\end{center}
\vskip -0.1in
\end{table*}

\begin{figure}[!ht]
\centering
  \includegraphics[width=0.5\textwidth]{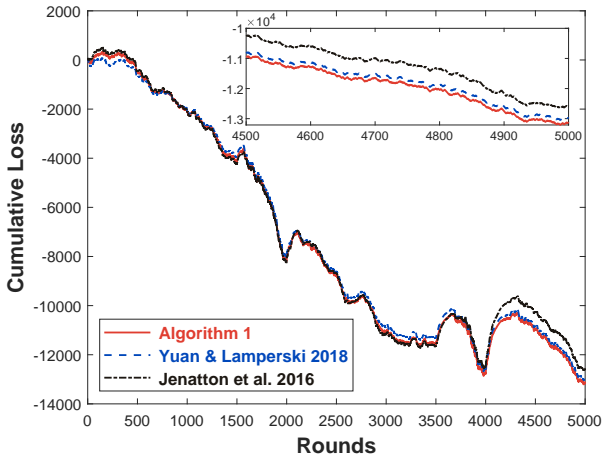}
  \caption{Evolutions of the cumulative loss in online quadratic programming.}
  \label{online_op:fig:lossSC}
\end{figure}

\begin{figure}[!ht]
\centering
  \includegraphics[width=0.5\textwidth]{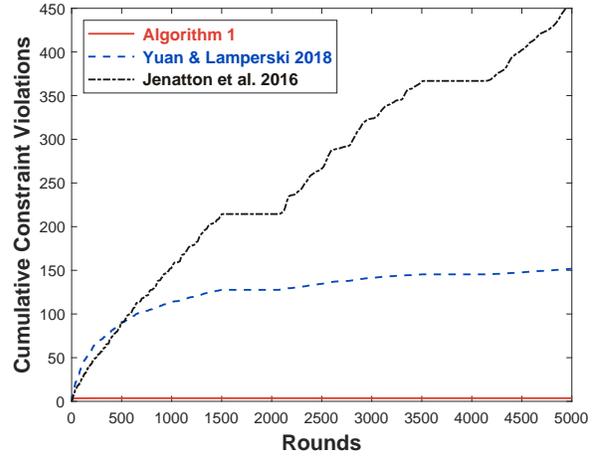}
  \caption{Evolutions of the cumulative constraint violation in online quadratic programming.}
  \label{online_op:fig:conSC}
\end{figure}

\section{Conclusions}\label{online_opsec:conclusion}
In this paper, we proposed two algorithms for online convex optimization with long term constraints. We analyzed static regret and cumulative constraint violation bounds for the first algorithm when the loss functions are convex and strongly convex, respectively. We achieved improved performance compared with existing results in the sense that the cumulative constraint violation is a stricter  metric than the commonly used constraint violation metric in the literature and smaller cumulative constraint violation bounds are achieved. We also analyzed regret with respect to any comparator sequence for both algorithms and the optimal $\mathcal{O}(\sqrt{T(1+P_T)})$ regret can be achieved by the second algorithm. In the future, we will investigate how to use the curvature of loss functions to reduce the regret and cumulative constraint violation bounds.

\section*{Acknowledgements}
The authors thank anonymous reviewers for the valuable comments and suggestions, and also thank the meta-reviewer for handling the review of this work.

This work was supported by Knut and Alice Wallenberg Foundation;  Swedish Foundation for Strategic Research; Swedish Research Council; Ministry of Education of Republic of Singapore under Grant AcRF TIER 1- 2019-T1-001-088 (RG72/19); National Natural Science Foundation of China under Grants 61991403, 61991404, 61991400, and 62003243; 2020 Science
and Technology Major Project of Liaoning Province under Grant 2020JH1/10100008; Shanghai Municipal Commission of Science and Technology under No. 19511132101; and Shanghai Municipal
Science and Technology Major Project under No. 2021SHZDZX0100.


\bibliography{refs}

\begin{thebibliography}{35}
\providecommand{\natexlab}[1]{#1}
\providecommand{\url}[1]{\texttt{#1}}
\expandafter\ifx\csname urlstyle\endcsname\relax
  \providecommand{\doi}[1]{doi: #1}\else
  \providecommand{\doi}{doi: \begingroup \urlstyle{rm}\Url}\fi

\bibitem[Agarwal et~al.(2010)Agarwal, Dekel, and Xiao]{agarwal2010optimal}
Agarwal, A., Dekel, O., and Xiao, L.
\newblock Optimal algorithms for online convex optimization with multi-point
  bandit feedback.
\newblock In \emph{Conference on Learning Theory}, pp.\  28--40, 2010.

\bibitem[Cesa-Bianchi et~al.(1996)Cesa-Bianchi, Long, and
  Warmuth]{cesa1996worst}
Cesa-Bianchi, N., Long, P.~M., and Warmuth, M.~K.
\newblock Worst-case quadratic loss bounds for prediction using linear
  functions and gradient descent.
\newblock \emph{IEEE Transactions on Neural Networks}, 7\penalty0 (3):\penalty0
  604--619, 1996.

\bibitem[Chen et~al.(2017)Chen, Ling, and Giannakis]{chen2017online}
Chen, T., Ling, Q., and Giannakis, G.~B.
\newblock An online convex optimization approach to proactive network resource
  allocation.
\newblock \emph{IEEE Transactions on Signal Processing}, 65\penalty0
  (24):\penalty0 6350--6364, 2017.

\bibitem[Crammer et~al.(2006)Crammer, Dekel, Keshet, Shalev-Shwartz, and
  Singer]{crammer2006online}
Crammer, K., Dekel, O., Keshet, J., Shalev-Shwartz, S., and Singer, Y.
\newblock Online passive aggressive algorithms.
\newblock \emph{Journal of Machine Learning Research}, 7:\penalty0 551--585,
  2006.

\bibitem[Gentile \& Warmuth(1999)Gentile and Warmuth]{gentile1999linear}
Gentile, C. and Warmuth, M.~K.
\newblock Linear hinge loss and average margin.
\newblock In \emph{Advances in Neural Information Processing Systems}, pp.\
  225--231, 1999.

\bibitem[Goldfarb \& Tucker(2011)Goldfarb and Tucker]{goldfarb2011online}
Goldfarb, A. and Tucker, C.
\newblock Online display advertising: Targeting and obtrusiveness.
\newblock \emph{Marketing Science}, 30\penalty0 (3):\penalty0 389--404, 2011.

\bibitem[Gordon(1999)]{gordon1999regret}
Gordon, G.~J.
\newblock Regret bounds for prediction problems.
\newblock In \emph{Conference on Learning Theory}, pp.\  29--40, 1999.

\bibitem[Hazan(2016)]{hazan2016introduction}
Hazan, E.
\newblock Introduction to online convex optimization.
\newblock \emph{Foundations and Trends in Optimization}, 2\penalty0
  (3-4):\penalty0 157--325, 2016.

\bibitem[Hazan et~al.(2007)Hazan, Agarwal, and Kale]{hazan2007logarithmic}
Hazan, E., Agarwal, A., and Kale, S.
\newblock Logarithmic regret algorithms for online convex optimization.
\newblock \emph{Machine Learning}, 69\penalty0 (2-3):\penalty0 169--192, 2007.

\bibitem[Jadbabaie et~al.(2015)Jadbabaie, Rakhlin, Shahrampour, and
  Sridharan]{jadbabaie2015online}
Jadbabaie, A., Rakhlin, A., Shahrampour, S., and Sridharan, K.
\newblock Online optimization: Competing with dynamic comparators.
\newblock In \emph{International Conference on Artificial Intelligence and
  Statistics}, pp.\  398--406, 2015.

\bibitem[Jenatton et~al.(2016)Jenatton, Huang, and
  Archambeau]{jenatton2016adaptive}
Jenatton, R., Huang, J., and Archambeau, C.
\newblock Adaptive algorithms for online convex optimization with long-term
  constraints.
\newblock In \emph{International Conference on Machine Learning}, pp.\
  402--411, 2016.

\bibitem[Li et~al.(2020)Li, Yi, and Xie]{Li2018distributed}
Li, X., Yi, X., and Xie, L.
\newblock Distributed online optimization for multi-agent networks with coupled
  inequality constraints.
\newblock \emph{IEEE Transactions on Automatic Control}, 2020.

\bibitem[Mahdavi et~al.(2012)Mahdavi, Jin, and Yang]{mahdavi2012trading}
Mahdavi, M., Jin, R., and Yang, T.
\newblock Trading regret for efficiency: Online convex optimization with long
  term constraints.
\newblock \emph{Journal of Machine Learning Research}, 13\penalty0
  (81):\penalty0 2503--2528, 2012.

\bibitem[Mokhtari et~al.(2016)Mokhtari, Shahrampour, Jadbabaie, and
  Ribeiro]{mokhtari2016online}
Mokhtari, A., Shahrampour, S., Jadbabaie, A., and Ribeiro, A.
\newblock Online optimization in dynamic environments: Improved regret rates
  for strongly convex problems.
\newblock In \emph{IEEE Conference on Decision and Control}, pp.\  7195--7201,
  2016.

\bibitem[Neely \& Yu(2017)Neely and Yu]{neely2017online}
Neely, M.~J. and Yu, H.
\newblock Online convex optimization with time-varying constraints.
\newblock \emph{arXiv:1702.04783}, 2017.

\bibitem[Shalev-Shwartz(2012)]{shalev2012online}
Shalev-Shwartz, S.
\newblock Online learning and online convex optimization.
\newblock \emph{Foundations and Trends in Machine Learning}, 4\penalty0
  (2):\penalty0 107--194, 2012.

\bibitem[Sun et~al.(2017)Sun, Dey, and Kapoor]{sun2017safety}
Sun, W., Dey, D., and Kapoor, A.
\newblock Safety-aware algorithms for adversarial contextual bandit.
\newblock In \emph{International Conference on Machine Learning}, pp.\
  3280--3288, 2017.

\bibitem[van Erven \& Koolen(2016)van Erven and Koolen]{NIPS2016_14cfdb59}
van Erven, T. and Koolen, W.~M.
\newblock Metagrad: Multiple learning rates in online learning.
\newblock In \emph{Advances in Neural Information Processing Systems}, pp.\
  3666--3674, 2016.

\bibitem[Yi et~al.(2020{\natexlab{a}})Yi, Li, Xie, and
  Johansson]{yi2020distributed}
Yi, X., Li, X., Xie, L., and Johansson, K.~H.
\newblock Distributed online convex optimization with time-varying coupled
  inequality constraints.
\newblock \emph{IEEE Transactions on Signal Processing}, 68:\penalty0 731--746,
  2020{\natexlab{a}}.

\bibitem[Yi et~al.(2020{\natexlab{b}})Yi, Li, Yang, Xie, Johansson, and
  Chai]{yi2019distributed}
Yi, X., Li, X., Yang, T., Xie, L., Johansson, K.~H., and Chai, T.
\newblock Distributed bandit online convex optimization with time-varying
  coupled inequality constraints.
\newblock \emph{IEEE Transactions on Automatic Control}, 2020{\natexlab{b}}.

\bibitem[Yi et~al.(2021)Yi, Li, Yang, Xie, Chai, and Johansson]{yi2021regret}
Yi, X., Li, X., Yang, T., Xie, L., Chai, T., and Johansson, K.~H.
\newblock Regret and cumulative constraint violation analysis for distributed
  online constrained convex optimization.
\newblock \emph{arXiv preprint arXiv:2105.00321}, 2021.

\bibitem[Yu \& Neely(2020)Yu and Neely]{yu2020lowJMLR}
Yu, H. and Neely, M.~J.
\newblock A low complexity algorithm with $ {O}(\sqrt{T})$ regret and $ {O}(1)$
  constraint violations for online convex optimization with long term
  constraints.
\newblock \emph{Journal of Machine Learning Research}, 21\penalty0
  (1):\penalty0 1--24, 2020.

\bibitem[Yu et~al.(2017)Yu, Neely, and Wei]{yu2017online}
Yu, H., Neely, M., and Wei, X.
\newblock Online convex optimization with stochastic constraints.
\newblock In \emph{Advances in Neural Information Processing Systems}, pp.\
  1428--1438, 2017.

\bibitem[Yuan et~al.(2021{\natexlab{a}})Yuan, Proutiere, and
  Shi]{yuan2021distributed}
Yuan, D., Proutiere, A., and Shi, G.
\newblock Distributed online linear regression.
\newblock \emph{IEEE Transactions on Information Theory}, 67\penalty0
  (1):\penalty0 616--639, 2021{\natexlab{a}}.

\bibitem[Yuan et~al.(2021{\natexlab{b}})Yuan, Proutiere, and
  Shi]{yuan2021distributedtac}
Yuan, D., Proutiere, A., and Shi, G.
\newblock Distributed online optimization with long-term constraints.
\newblock \emph{IEEE Transactions on Automatic Control}, 2021{\natexlab{b}}.

\bibitem[Yuan \& Lamperski(2018)Yuan and Lamperski]{NIPS2018_7852}
Yuan, J. and Lamperski, A.
\newblock Online convex optimization for cumulative constraints.
\newblock In \emph{Advances in Neural Information Processing Systems}, pp.\
  6140--6149, 2018.

\bibitem[Zhang(2020)]{ijcai2020-731}
Zhang, L.
\newblock Online learning in changing environments.
\newblock In \emph{International Joint Conference on Artificial Intelligence},
  pp.\  5178--5182, 2020.

\bibitem[Zhang et~al.(2017)Zhang, Yang, Yi, Rong, and Zhou]{zhang2017improved}
Zhang, L., Yang, T., Yi, J., Rong, J., and Zhou, Z.-H.
\newblock Improved dynamic regret for non-degenerate functions.
\newblock In \emph{Advances in Neural Information Processing Systems}, pp.\
  732--741, 2017.

\bibitem[Zhang et~al.(2018{\natexlab{a}})Zhang, Lu, and
  Zhou]{zhang2018adaptive}
Zhang, L., Lu, S., and Zhou, Z.-H.
\newblock Adaptive online learning in dynamic environments.
\newblock In \emph{Advances in Neural Information Processing Systems}, pp.\
  1323--1333, 2018{\natexlab{a}}.

\bibitem[Zhang et~al.(2018{\natexlab{b}})Zhang, Yang, Jin, and
  Zhou]{zhang2018dynamic}
Zhang, L., Yang, T., Jin, R., and Zhou, Z.-H.
\newblock Dynamic regret of strongly adaptive methods.
\newblock In \emph{International Conference on Machine Learning}, pp.\
  5882--5891, 2018{\natexlab{b}}.

\bibitem[Zhang et~al.(2019)Zhang, Liu, and Zhou]{pmlr-v97-zhang19}
Zhang, L., Liu, T.-Y., and Zhou, Z.-H.
\newblock Adaptive regret of convex and smooth functions.
\newblock In \emph{International Conference on Machine Learning}, pp.\
  7414--7423, 2019.

\bibitem[Zhang et~al.(2020)Zhang, Zhao, and Zhou]{pmlr-v124-zhang20a}
Zhang, Y.-J., Zhao, P., and Zhou, Z.-H.
\newblock A simple online algorithm for competing with dynamic comparators.
\newblock In \emph{Conference on Uncertainty in Artificial Intelligence}, pp.\
  390--399, 2020.

\bibitem[Zhao \& Zhang(2020)Zhao and Zhang]{zhao2020improved}
Zhao, P. and Zhang, L.
\newblock Improved analysis for dynamic regret of strongly convex and smooth
  functions.
\newblock \emph{arXiv:2006.05876}, 2020.

\bibitem[Zhao et~al.(2020)Zhao, Zhang, Zhang, and Zhou]{zhao2020dynamic}
Zhao, P., Zhang, Y.-J., Zhang, L., and Zhou, Z.-H.
\newblock Dynamic regret of convex and smooth functions.
\newblock In \emph{Advances in Neural Information Processing Systems}, 2020.

\bibitem[Zinkevich(2003)]{zinkevich2003online}
Zinkevich, M.
\newblock Online convex programming and generalized infinitesimal gradient
  ascent.
\newblock In \emph{International Conference on Machine Learning}, pp.\
  928--936, 2003.

\end{thebibliography}
\bibliographystyle{icml2021}

\appendix

\section{Useful Lemmas}\label{online_op:app-lemmas}

The following results are used in the proofs.

\begin{lemma}\label{online_op:lemma_mirror}
Suppose that  $h:\mathbb{S}\rightarrow\mathbb{R}$ is a convex function with $\mathbb{S}$ being a convex and closed set in $\mathbb{R}^p$. Moreover, assume that $\partial h(x),~\forall x\in\mathbb{S}$, exists. Given $z\in\mathbb{S}$,
the  projection
\begin{align*}
y=\argmin_{x\in\mathbb{S}}\{h(x)+\|x-z\|^2\},
\end{align*}
satisfies
\begin{align*}
&\langle y-x,\partial h(y)\rangle\nonumber\\
&\le \|x-z\|^2-\|x-y\|^2-\|y-z\|^2,~\forall x\in\mathbb{S}.
\end{align*}
\end{lemma}
This lemma is a special case of Lemma~1 in \citet{yi2020distributed}.

\begin{lemma}\label{online_op:lemma_expert}
Suppose that  $\{h_t:\mathbb{S}\rightarrow\mathbb{R}\}$ is a sequence of convex functions with $\mathbb{S}$ being a convex set in $\mathbb{R}^p$. Assume there exists a constant $F_h>0$ such that $|h_t(x)|\le F_h,~\forall x\in\mathbb{S},~t\in\mathbb{N}_+$. Let $N\in\mathbb{N}_+$ and $\beta>0$ be constants. For each $i\in[N]$, let $\{x_{i,t}\}$ be a sequence in $\mathbb{S}$. Then, for any given $w_{i,1}\in(0,1)$ satisfying $\sum_{i=1}^{N}w_{i,1}=1$, the sequence $\{x_t\}$ generated by
\begin{align*}
w_{i,t}&=\frac{w_{i,t-1}e^{-\beta h_{t-1}(x_{i,t-1})}}{\sum_{i=1}^{N}w_{i,t-1}e^{-\beta h_{t-1}(x_{i,t-1})}},\\
x_{t}&=\sum_{i=1}^{N}w_{i,t}x_{i,t}.
\end{align*}
satisfies
\begin{align*}
\sum_{t=1}^{T}h_t(x_t)-\min_{i\in[N]}\Big\{\sum_{t=1}^{T}h_t(x_{i,t})+\frac{1}{\beta}\ln\frac{1}{w_{i,1}}\Big\}\le\frac{\beta F_h^2T}{2}.
\end{align*}
\end{lemma}
The proof of this lemma follows the proof of Lemma~1 in \citet{zhang2018adaptive}.

\section{Proof of Lemma~\ref{online_op:theoremreg_alg2}}\label{online_op:theoremreg_alg2proof}
To prove Lemma~\ref{online_op:theoremreg_alg2}, we need the following result.

\begin{lemma}\label{online_op:lemma_regretdelta_alg2}
Suppose Assumptions~\ref{online_op:assfunction}--\ref{online_op:ass_subgradient} hold. Let $\{x_{t}\}$ be the sequence generated by Algorithm \ref{online_op:algorithm2} and $\{y_{t}\}$ be an arbitrary sequence in $\mathbb{X}$, then
\begin{align}\label{online_op:lemma_regretdeltaequ_alg2}
&\gamma_{t+1}q_{t}^\top [g(x_{t+1})]_+
+f_{t}(x_{t})-f_{t}(y_t)\nonumber\\
&\le \gamma_{t+1}\hat{q}_{t}^\top [g(y_{t})]_+
+\Delta_{t}(y_t)-\frac{1}{2\alpha_{t}}\|x_{t}-x_{t+1}\|^2\nonumber\\
&\quad+\frac{G^2\alpha_{t}}{2}
-\gamma_{t+1}\gamma_{t}([g(x_{t})]_+)^\top [g(x_{t+1})]_+.
\end{align}
\end{lemma}
\begin{proof}

From $q_{0}=\bm{0}_{m}$, \eqref{online_op:al_q2}, and \eqref{online_op:al_qhat2}, it is straightforward to check that
\begin{align}\label{online_op:lemma_virtual_equ2_alg2}
q_{t}\ge\bm{0}_{m},~\hat{q}_{t}\ge\bm{0}_{m},~t\in\mathbb{N}_+.
\end{align}

From (\ref{online_op:subgradient}) and \eqref{online_op:subgupper}, we have
\begin{align}\label{online_op:fxy}
&f_{t}(x_{t})-f_{t}(y_t)
\le\langle\partial f_{t}(x_{t}),x_{t}-y_t\rangle\nonumber\\
&=\langle\partial f_{t}(x_{t}),x_{t}-x_{t+1}\rangle
+\langle\partial f_{t}(x_{t}),x_{t+1}-y_t\rangle\nonumber\\
&\le G\|x_{t}-x_{t+1}\|
+\langle\partial f_{t}(x_{t}),x_{t+1}-y_t\rangle\nonumber\\
&\le\frac{G^2\alpha_{t}}{2}+\frac{1}{2\alpha_{t}}\|x_{t}-x_{t+1}\|^2
+\langle\partial f_{t}(x_{t}),x_{t+1}-y_t\rangle.
\end{align}
For the second term of \eqref{online_op:fxy}, we have
\begin{align}
&\langle\partial f_{t}(x_{t}),x_{t+1}-y_t\rangle\nonumber\\
&=\gamma_{t+1}\langle(\partial [g(x_{t+1})]_+)^\top \hat{q}_{t},y_t-x_{t+1}\rangle\nonumber\\
&\quad+\langle\omega_{t+1},x_{t+1}-y_t\rangle,\label{online_op:fxy1}
\end{align}
where $\omega_{t+1}=\partial f_{t}(x_{t})+\gamma_{t+1}(\partial [g(x_{t+1})]_+)^\top \hat{q}_{t}$.

We next to find the upper bound of each term in the right-hand side of \eqref{online_op:fxy1}.

From $g(\cdot)$ is convex, it is straightforward to see that $[g(\cdot)]_+$ is also convex.
From (\ref{online_op:subgradient}) and \eqref{online_op:lemma_virtual_equ2_alg2}, we have
\begin{align}
&\gamma_{t+1}\langle(\partial [g(x_{t+1})]_+)^\top \hat{q}_{t},y_t-x_{t+1}\rangle\nonumber\\
&\le \gamma_{t+1}\hat{q}_{t}^\top [g(y_{t})]_+ -\gamma_{t+1}\hat{q}_{t}^\top [g(x_{t+1})]_+\nonumber\\
&= \gamma_{t+1}\hat{q}_{t}^\top [g(y_{t})]_+ -\gamma_{t+1}q_{t}^\top [g(x_{t+1})]_+\nonumber\\
&\quad-\gamma_{t+1}\gamma_{t}([g(x_{t})]_+)^\top [g(x_{t+1})]_+
.\label{online_op:gyxdelta}
\end{align}

Applying Lemma~\ref{online_op:lemma_mirror} to the update (\ref{online_op:al_x2}), we get
\begin{align}
&\langle\omega_{t+1},x_{t+1}-y_t\rangle\nonumber\\
&\le\frac{1}{\alpha_{t}}(\|y_t-x_{t}\|^2-\|y_t-x_{t+1}\|^2-\|x_{t+1}-x_{t}\|^2).
\label{online_op:omgea2}
\end{align}

Combining (\ref{online_op:fxy})--(\ref{online_op:omgea2}) and rearranging terms yields (\ref{online_op:lemma_regretdeltaequ_alg2}).
\end{proof}

We are now ready to prove Lemma~\ref{online_op:theoremreg_alg2}.

\noindent {\bf (i)} Noting that $g(y_t)\le{\bm 0}_{m}$ when $y_{[T]}\in\calX^{T}$,  and summing (\ref{online_op:lemma_regretdeltaequ_alg2}) over $t\in[T]$ gives \eqref{online_op:theoremregequ_alg2}.

\noindent {\bf (ii)}   From \eqref{online_op:al_q2} and \eqref{online_op:lemma_virtual_equ2_alg2}, we have
\begin{align*}
\|q_{t}\|_1&=\|q_{t-1}+\gamma_t[g(x_{t})]_+\|_1\\
&=\|q_{t-1}\|_1+\gamma_t\|[g(x_{t})]_+\|_1,
\end{align*}
which yields
\begin{align}\label{online_op:lemma_virtual_equ4_alg2}
&\|[g(x_{t})]_+\|_1
=\frac{1}{\gamma_t}(\|q_{t}\|_1-\|q_{t-1}\|_1)\nonumber\\
&=\frac{1}{\gamma_t}\|q_{t}\|_1-\frac{1}{\gamma_{t-1}}\|q_{t-1}\|_1
+\Big(\frac{1}{\gamma_{t-1}}-\frac{1}{\gamma_t}\Big)\|q_{t-1}\|_1.
\end{align}
Summing \eqref{online_op:lemma_virtual_equ4_alg2} over $t\in[T]$ yields
\begin{align*}
\sum_{t=1}^{T}\|[g(x_{t})]_+\|_1=\frac{1}{\gamma_T}\|q_{T}\|_1
+\sum_{t=1}^{T-1}\Big(\frac{1}{\gamma_{t}}-\frac{1}{\gamma_{t+1}}\Big)\|q_{t}\|_1.
\end{align*}
From the above inequality, $\|x\|\le\|x\|_1\le\sqrt{m}\|x\|,~\forall x\in\mathbb{R}^m$, and $\{\gamma_t\}$ is non-decreasing, we have \eqref{online_op:theoremconsequ_alg2}.

\noindent {\bf (iii)}   From \eqref{online_op:al_q2}, we have
\begin{align*}
\|q_{t}\|^2&=\|q_{t-1}+\gamma_{t}[g(x_{t})]_+\|^2\\
&=\|q_{t-1}\|^2+2\gamma_{t}q_{t-1}^\top[g(x_{t})]_+ +\|\gamma_{t}[g(x_{t})]_+\|^2,
\end{align*}
which implies
\begin{align}
&\frac{1}{2}(\|q_{t}\|^2-\|q_{t-1}\|^2)\nonumber\\
&=\gamma_{t}q_{t-1}^\top[g(x_{t})]_+ +\frac{1}{2}\|\gamma_{t}[g(x_{t})]_+\|^2.\label{online_op:lemma_virtual_equ5_alg2}
\end{align}

From Assumptions~\ref{online_op:assfunction} and \ref{online_op:ass_subgradient}, and Lemma~2.6 in \citet{shalev2012online}, it follows that
\begin{align}
\left\|g(x)-g(y)\right\|\le G\|x-y\|,~\forall x,~y\in \mathbb{X}.\label{online_op:assfunction:functionLipg}
\end{align}

We have
\begin{align}
&\gamma_{t+1}\gamma_{t}([g(x_{t})]_+)^\top [g(x_{t+1})]_+\nonumber\\
&=\frac{1}{2}(\|\gamma_{t}[g(x_{t})]_+\|^2 +\|\gamma_{t+1}[g(x_{t+1})]_+\|^2\nonumber\\
&\quad-\|\gamma_{t+1}[g(x_{t+1})]_+-\gamma_{t}[g(x_{t})]_+\|^2)\nonumber\\
&\ge\frac{1}{2}(\|\gamma_{t}[g(x_{t})]_+\|^2 +\|\gamma_{t+1}[g(x_{t+1})]_+\|^2)\nonumber\\
&\quad-\|\gamma_{t}[g(x_{t+1})]_+-\gamma_{t}[g(x_{t})]_+\|^2\nonumber\\
&\quad-\|(\gamma_{t+1}-\gamma_{t})[g(x_{t+1})]_+\|^2\nonumber\\
&\ge\frac{1}{2}(\|\gamma_{t}[g(x_{t})]_+\|^2 +\|\gamma_{t+1}[g(x_{t+1})]_+\|^2)\nonumber\\
&\quad-\gamma_{t}^2G^2\|x_{t+1}-x_{t}\|^2
-(\gamma_{t+1}-\gamma_{t})^2F^2
,\label{online_op:gxxtilde_alg2}
\end{align}
where the last inequality holds since \eqref{online_op:assfunction:functionLipg},  \eqref{online_op:ftgtupper}, and that the projection $[\cdot]_+$ is nonexpansive, i.e.,
\begin{align*}
\|[x]_+-[y]_+\|\le\|x-y\|,~\forall x,y\in\mathbb{R}^p.
\end{align*}

Noting that $g(y_t)\le{\bm 0}_{m}$ when $y_{[T]}\in\calX^{T}$, combining \eqref{online_op:lemma_virtual_equ5_alg2}, (\ref{online_op:lemma_regretdeltaequ_alg2}), and \eqref{online_op:gxxtilde_alg2}, and summing over $t\in[T]$ gives
\begin{align}\label{online_op:theoremregequ2_alg2_g}
&\frac{1}{2}\|q_{T+1}\|^2
+\frac{1}{2}\sum_{t=1}^T\|\gamma_{t}[g(x_{t})]_+\|^2\nonumber\\
&\le \frac{1}{2}\|q_{1}\|^2+\sum_{t=1}^T\Delta_{t}(y_t)+\sum_{t=1}^T\frac{G^2\alpha_{t}}{2}\nonumber\\
&\quad
+\sum_{t=1}^T(\gamma_{t+1}-\gamma_{t})^2F^2-\Reg(x_{[T]},y_{[T]})\nonumber\\
&\quad
+\sum_{t=1}^T\Big(\gamma_{t}^2G^2-\frac{1}{2\alpha_{t}}\Big)\|x_{t+1}-x_{t}\|^2.
\end{align}

From \eqref{online_op:ftgtupper}, we have
\begin{align}\label{online_op:ff}
-\Reg(x_{[T]},y_{[T]})\le FT.
\end{align}

 From \eqref{online_op:al_q2} and \eqref{online_op:ftgtupper}, we have
\begin{align}\label{online_op:ffq}
\frac{1}{2}\|q_{1}\|^2\le \frac{1}{2}\gamma_1^2F^2.
\end{align}

Combining (\ref{online_op:theoremregequ2_alg2_g})--\eqref{online_op:ffq}, and noting $\gamma_{t}^2=\gamma_0^2/\alpha_{t}$ with $\gamma_0\in(0,1/(\sqrt{2}G)]$ yields \eqref{online_op:theoremregequ_alg2_g}.

\section{Proof of Theorem~\ref{online_op:corollaryreg}}\label{online_op:corollaryregproof}

\noindent {\bf (i)} Using \eqref{online_op:stepsize1} and setting $y_t=\check{x}^*_T$ yields
\begin{align}
\sum_{t=1}^T\Delta_{t}(\check{x}^*_T)
&=\frac{T^c}{\alpha_0}\sum_{t=1}^T(\|\check{x}^*_T-x_{t}\|^2-\|\check{x}^*_T-x_{t+1}\|^2)\nonumber\\
&\le \frac{T^c}{\alpha_0}\|\check{x}^*_T-x_{1}\|^2.\label{online_op:dyz_alg2}
\end{align}

Setting $y_t=\check{x}^*_T$, and combining  (\ref{online_op:theoremregequ_alg2}), \eqref{online_op:stepsize1}, and (\ref{online_op:dyz_alg2}) yields
\begin{align*}
\Reg(x_{[T]},\check{x}^*_{[T]})\le \frac{T^c}{\alpha_0}\|\check{x}^*_T-x_{1}\|^2+\frac{G^2\alpha_0}{2}T^{1-c},
\end{align*}
which gives \eqref{online_op:corollaryregequ1}.

\noindent {\bf (ii)} Setting $y_t=\check{x}^*_T$, and combining  (\ref{online_op:theoremregequ_alg2_g}), \eqref{online_op:stepsize1}, and (\ref{online_op:dyz_alg2}) yields
\begin{align}\label{online_op:corollaryconsequ_proof1}
\|q_{T}\|^2\le \varepsilon_1T,
\end{align}
where $\varepsilon_1=\frac{2}{\alpha_0T^{1-c}}\|\check{x}^*_T-x_{1}\|^2+\frac{G^2\alpha_0}{T^c}
+\frac{\gamma_0^2F^2}{\alpha_0T}+2F$.

From \eqref{online_op:stepsize1}, \eqref{online_op:theoremconsequ_alg2}, and \eqref{online_op:corollaryconsequ_proof1}, we have
\begin{align}\label{online_op:corollaryconsequ_proof2}
\sum_{t=1}^{T} \|[g(x_{t})]_+\|
\le\frac{\sqrt{m\alpha_0}}{\gamma_0 T^{c/2}}\|q_{T}\|
\le\frac{\sqrt{m\alpha_0\varepsilon_1}}{\gamma_0}T^{(1-c)/2},
\end{align}
which yields \eqref{online_op:corollaryconsequ}.

\section{Proof of Corollary~\ref{online_op:corollaryreg_sc}}\label{online_op:corollaryregproof_sc}
\noindent {\bf (i)}  From \eqref{online_op:assstrongconvexequ}, we know that \eqref{online_op:fxy} can be replaced by
\begin{align}\label{online_op:fxy_sc}
&f_{t}(x_{t})-f_{t}(y_t)\nonumber\\
&\le \frac{G^2\alpha_{t}}{2}+\frac{1}{2\alpha_{t}}\|x_{t}-x_{t+1}\|^2-\mu\|x_t-y_t\|^2\nonumber\\
&\quad+\langle\partial f_{t}(x_{t}),x_{t+1}-y_t\rangle.
\end{align}
Note that compared with \eqref{online_op:fxy}, \eqref{online_op:fxy_sc} has an extra term $-\mu\|x_t-y_t\|^2$. Then, \eqref{online_op:theoremregequ_alg2} can be replaced by
\begin{align}
&\Reg(x_{[T]},y_{[T]})\nonumber\\
&\le \sum_{t=1}^T(\Delta_{t}(y_t)-\mu\|x_t-y_t\|^2)+\sum_{t=1}^{T}\frac{G^2}{2}\alpha_{t}.\label{online_op:theoremregequ_alg2_sc}
\end{align}

Using \eqref{online_op:stepsize2} and setting $y_t=\check{x}^*_T$ yields
\begin{align}
&\sum_{t=1}^T(\Delta_{t}(\check{x}^*_T)-\mu\|x_t-\check{x}^*_T\|^2)\nonumber\\
&=\sum_{t=1}^T
((t-1)\mu\|x_t-\check{x}^*_T\|^2-t\mu\|x_{t+1}-\check{x}^*_T\|^2)\nonumber\\
&\le0.\label{online_op:dyz_sc}
\end{align}

From \eqref{online_op:stepsize2}, we have
\begin{align}
\sum_{t=1}^T\alpha_{t}=\sum_{t=1}^T\frac{1}{t\mu}\le\int_{1}^{T+1}\frac{1}{t\mu}dt
=\frac{\log(T+1)}{\mu}.
\label{online_op:corollaryregequ11}
\end{align}

From \eqref{online_op:theoremregequ_alg2_sc}--\eqref{online_op:corollaryregequ11}, we have
\begin{align*}
\Reg(x_{[T]},\check{x}^*_{[T]})\le \frac{G^2\log(T+1)}{2\mu},
\end{align*}
which gives \eqref{online_op:corollaryregequ1_sc}.

\noindent {\bf (ii)} Noting \eqref{online_op:fxy_sc}, we can replace \eqref{online_op:theoremregequ_alg2_g} by
\begin{align}
\frac{1}{2}\|q_{T+1}\|^2&\le  \sum_{t=1}^T(\Delta_{t}(y_t)-\mu\|x_t-y_t\|^2)
+\sum_{t=1}^{T}\frac{G^2\alpha_{t}}{2}\nonumber\\
&\quad+\frac{1}{2}\gamma_1^2F^2+FT
+\sum_{t=1}^T(\gamma_{t+1}-\gamma_{t})^2F^2.\label{online_op:theoremregequ_alg2_g_sc}
\end{align}

From \eqref{online_op:stepsize2}, we have
\begin{align}
&\sum_{t=1}^T(\gamma_{t+1}-\gamma_{t})^2
=\sum_{t=1}^T\gamma_0^2\mu(\sqrt{t+1}-\sqrt{t})^2\nonumber\\
&\le\sum_{t=1}^T\gamma_0^2\mu\Big(\frac{1}{2\sqrt{t}}\Big)^2
\le\frac{\gamma_0^2\mu}{4}\log(T+1).
\label{online_op:corollaryregequ12}
\end{align}

Setting $y_t=\check{x}^*_T$ and combining \eqref{online_op:dyz_sc}--\eqref{online_op:corollaryregequ12} yields
\begin{align}\label{online_op:corollaryconsequ_proof1_sc}
\|q_{T}\|^2\le \varepsilon_2T,~\forall T\in\mathbb{N}_+,
\end{align}
where $\varepsilon_2=\frac{G^2}{\mu}+\gamma_0^2F^2\mu +2F+\frac{\gamma_0^2F^2\mu}{2}$.

From \eqref{online_op:stepsize2}, \eqref{online_op:corollaryconsequ_proof1_sc}, and \eqref{online_op:theoremconsequ_alg2}, we have
\begin{align*}
&\sum_{t=1}^{T} \|[g(x_{t})]_+\|\nonumber\\
&\le\frac{\sqrt{m}}{\gamma_0\sqrt{\mu}}\Big(\frac{1}{\sqrt{T}}\|q_{T}\|
+\sum_{t=1}^{T-1}\Big(\frac{1}{\sqrt{t}}-\frac{1}{\sqrt{t+1}}\Big)
\|q_{t}\|\Big)\nonumber\\
&\le\frac{\sqrt{m\varepsilon_2}}{\gamma_0\sqrt{\mu}}\Big(1+\sum_{t=1}^{T-1}
\Big(\frac{1}{\sqrt{t}}-\frac{1}{\sqrt{t+1}}\Big)\sqrt{t}\Big)\nonumber\\
&=\frac{\sqrt{m\varepsilon_2}}{\gamma_0\sqrt{\mu}}\Big(1+\sum_{t=1}^{T-1}
\Big(\Big(\frac{1}{t}+1\Big)^{\frac{1}{2}}-1\Big)\frac{\sqrt{t}}{\sqrt{t+1}}\Big)\nonumber\\
&\le\frac{\sqrt{m\varepsilon_2}}{\gamma_0\sqrt{\mu}}\Big(1+\sum_{t=1}^{T-1}
\frac{1}{2t}\Big)\nonumber\\
&\le\frac{\sqrt{m\varepsilon_2}}{\gamma_0\sqrt{\mu}}\Big(1+
\frac{1}{2}\log(T)\Big),
\end{align*}
which yields \eqref{online_op:corollaryconsequ_sc}.

\section{Proof of Theorem~\ref{online_op:corollaryreg_dr}}\label{online_op:corollaryregproof_dr}
\noindent {\bf (i)}  Using \eqref{online_op:domainupper} and \eqref{online_op:stepsize1_dr}  yields
\begin{align}
&\sum_{t=1}^T\Delta_{t}(y_t)
=\sum_{t=1}^T \frac{t^c}{\alpha_0}(\|y_t-x_{t}\|^2-\|y_t-x_{t+1}\|^2)\nonumber\\
&=\frac{1}{\alpha_0}\sum_{t=1}^T(t^c \|y_t-x_{t}\|^2-(t+1)^c\|y_{t+1}-x_{t+1}\|^2\nonumber\\
&\quad+(t+1)^c\|y_{t+1}-x_{t+1}\|^2-t^c\|y_{t+1}-x_{t+1}\|^2\nonumber\\
&\quad+t^c\|y_{t+1}-x_{t+1}\|^2-t^c\|y_{t}-x_{t+1}\|^2)\nonumber\\
&\le \frac{1}{\alpha_0}\|y_1-x_{1}\|^2
+\frac{1}{\alpha_0}\sum_{t=1}^T((t+1)^c-t^c)(d(\mathbb{X}))^2\nonumber\\
&\quad+\frac{2}{\alpha_0}\sum_{t=1}^Tt^cd(\mathbb{X})\|y_{t+1}-y_t\|\nonumber\\
&\le \frac{1}{\alpha_0}(1+(T+1)^c-1)(d(\mathbb{X}))^2
+\frac{2T^cd(\mathbb{X})}{\alpha_0}P_T\nonumber\\
&\le \frac{2}{\alpha_0}(d(\mathbb{X}))^2T^c\Big(1+\frac{P_T}{d(\mathbb{X})}\Big).
\label{online_op:dyz_alg2_dr}
\end{align}

From \eqref{online_op:stepsize1_dr}, we have
\begin{align}
\sum_{t=1}^T\alpha_{t}&=\sum_{t=1}^T\frac{\alpha_0}{t^c}\le\int_1^T\frac{\alpha_0}{t^c}dt+\alpha_0\nonumber\\
&=\frac{\alpha_0(T^{1-c}-c)}{1-c}\le\frac{\alpha_0T^{1-c}}{1-c}.
\label{online_op:corollaryregequ11_dr}
\end{align}

From  (\ref{online_op:theoremregequ_alg2}), \eqref{online_op:dyz_alg2_dr}, and (\ref{online_op:corollaryregequ11_dr}), we have
\begin{align}\label{online_op:corollaryregequ1_proof_dr}
&\Reg(x_{[T]},y_{[T]})\nonumber\\&\le \frac{2}{\alpha_0}(d(\mathbb{X}))^2T^c\Big(1+\frac{P_T}{d(\mathbb{X})}\Big)
+\frac{G^2\alpha_0}{2(1-c)}T^{1-c},
\end{align}
which gives \eqref{online_op:corollaryregequ1_dr}.

\noindent {\bf (ii)} From \eqref{online_op:stepsize1_dr}, we have
\begin{align}
\sum_{t=1}^T(\gamma_{t+1}-\gamma_{t})^2F^2\le\sum_{t=1}^T(\gamma_{t+1}^2-\gamma_{t}^2)F^2
\le\frac{\gamma_0^2F^2}{\alpha_0}T^c.
\label{online_op:corollaryregequ12_dr}
\end{align}

 Using \eqref{online_op:dyz_alg2_dr} and setting $y_t=\check{x}^*_T$ yields
\begin{align}
\sum_{t=1}^T\Delta_{t}(\check{x}^*_T)
\le \frac{2}{\alpha_0}(d(\mathbb{X}))^2T^c.\label{online_op:dyz_alg2_dr2}
\end{align}

Noting that (\ref{online_op:theoremregequ_alg2_g}) holds for all $y_{[T]}\in\calX^{T}$, setting $y_t=\check{x}^*_T$, and combining \eqref{online_op:domainupper}, (\ref{online_op:theoremregequ_alg2_g}), (\ref{online_op:dyz_alg2_dr2}), (\ref{online_op:corollaryregequ11_dr}), and \eqref{online_op:corollaryregequ12_dr},   yields
\begin{align}\label{online_op:corollaryconsequ_proof1_dr}
\|q_{T}\|^2\le \varepsilon_3T,
\end{align}
where $\varepsilon_3=\frac{4(d(\mathbb{X}))^2}{\alpha_0T^{1-c}}+\frac{G^2\alpha_0}{(1-c)T^c}
+\frac{\gamma_0^2F^2}{\alpha_0T}+\frac{2\gamma_0^2F^2}{\alpha_0T^{1-c}}+2F$.

From  \eqref{online_op:stepsize1_dr}, \eqref{online_op:corollaryconsequ_proof1_dr}, and \eqref{online_op:theoremconsequ_alg2}, we have
\begin{align}\label{online_op:corollaryconsequ_proof2_dr}
&\sum_{t=1}^{T} \|[g(x_{t})]_+\|\nonumber\\
&\le\frac{\sqrt{m\alpha_0}}{\gamma_0}\Big(\frac{1}{T^{c/2}}\|q_{T}\|
+\sum_{t=1}^{T-1}\Big(\frac{1}{t^{c/2}}-\frac{1}{(t+1)^{c/2}}\Big)
\|q_{t}\|\Big)\nonumber\\
&\le\frac{\sqrt{m\alpha_0\varepsilon_3}}{\gamma_0}\Big(T^{(1-c)/2}+\sum_{t=1}^{T-1}
\Big(\frac{1}{t^{c/2}}-\frac{1}{(t+1)^{c/2}}\Big)\sqrt{t}\Big)\nonumber\\
&=\frac{\sqrt{m\alpha_0\varepsilon_3}}{\gamma_0}\Big(T^{(1-c)/2}\nonumber\\
&\quad+\sum_{t=1}^{T-1}
\Big(\Big(\frac{1}{t}+1\Big)^{c/2}-1\Big)\frac{\sqrt{t}}{(t+1)^{c/2}}\Big)\nonumber\\
&\le\frac{\sqrt{m\alpha_0\varepsilon_3}}{\gamma_0}\Big(T^{(1-c)/2}\nonumber\\
&\quad+\sum_{t=1}^{T-1}
\Big(\Big(\frac{1}{t}+1\Big)^{1/2}-1\Big)t^{(1-c)/2}\Big)\nonumber\\
&\le\frac{\sqrt{m\alpha_0\varepsilon_3}}{\gamma_0}\Big(T^{(1-c)/2}+\sum_{t=1}^{T-1}
\Big(1+\frac{1}{2t}-1\Big)t^{(1-c)/2}\Big)\nonumber\\
&=\frac{\sqrt{m\alpha_0\varepsilon_3}}{\gamma_0}\Big(T^{(1-c)/2}+\sum_{t=1}^{T-1}
\frac{1}{2t^{(1+c)/2}}\Big)\nonumber\\
&\le\frac{\sqrt{m\alpha_0\varepsilon_3}}{\gamma_0}\Big(T^{(1-c)/2}+\frac{1}{1-c}T^{(1-c)/2}\Big),
\end{align}
which yields \eqref{online_op:corollaryconsequ_dr}.

\section{Proof of Theorem~\ref{online_op:corollaryreg_alg1}}\label{online_op:corollaryregproof_alg1}
\noindent {\bf (i)} Noting that $|\ell_t(x)|\le Gd(\mathbb{X})$ and $\|\partial\ell_t(x)\|\le G$ due to Assumptions~\ref{online_op:ass_subgradient} and \ref{online_op:ass_set_bounded}, and that for each $i\in[N]$, the updating equations \eqref{online_op:al_q}--\eqref{online_op:al_xi} are exactly \eqref{online_op:al_q2}--\eqref{online_op:al_x2} for solving constrained online convex optimization with loss functions $\{\ell_t(x)\}$, similar to  get \eqref{online_op:corollaryregequ1_proof_dr} and \eqref{online_op:corollaryconsequ_proof2_dr}, for all $i\in[n]$ we have
\begin{align}
&\sum_{t=1}^{T}(\ell_t(x_{i,t})-\ell_t(y_t))\nonumber\\
&\le \frac{2}{\alpha_02^{i-1}}(d(\mathbb{X}))^2T^{c}\Big(1+\frac{P_T}{d(\mathbb{X})}\Big)
+G^2\alpha_02^{i-1}T^{1-c},\label{online_op:corollaryregequ1_proof_dr2}\\
&\sum_{t=1}^{T} \|[g(x_{i,t})]_+\|\nonumber\\
&\le\frac{\sqrt{m\alpha_02^{i-1}\varepsilon_{i,4}}}{\gamma_0}
\Big(T^{(1-c)/2}+\frac{1}{1-c}T^{(1-c)/2}\Big),
\label{online_op:corollaryconsequ_proof2_dr2}
\end{align}
where $\varepsilon_{i,4}=\frac{4(d(\mathbb{X}))^2}{\alpha_02^{i-1}T^{1-c}}+\frac{G^2\alpha_02^{i-1}}{(1-c)T^c}
+\frac{\gamma_0^2F^2}{\alpha_02^{i-1}T}+\frac{2\gamma_0^2F^2}{\alpha_02^{i-1}T^{1-c}}+2Gd(\mathbb{X})$.

From \eqref{online_op:domainupper}, we have
\begin{align*}
&1\le\Big(1+\frac{P_T}{d(\mathbb{X})}\Big)^\kappa\le(1+T)^\kappa,\\
&2^{N-1}\ge(1+T)^\kappa,
\end{align*}
which yield
\begin{align}
2^{i_0-1}\le\Big(1+\frac{P_T}{d(\mathbb{X})}\Big)^{\kappa}\le2^{i_0},\label{online_op:corollaryregequ12_alg1}
\end{align}
where $i_0=\lfloor \kappa\log_2(1+P_T/d(\mathbb{X}))\rfloor+1\in[N]$.

Combining \eqref{online_op:corollaryregequ1_proof_dr2} and \eqref{online_op:corollaryregequ12_alg1} yields
\begin{align}\label{online_op:corollaryregequ1_alg1_proof1}
\sum_{t=1}^{T}(\ell_t(x_{i_0,t})-\ell_t(y_t))
&\le  \frac{4}{\alpha_0}(d(\mathbb{X}))^2T^c\Big(1+\frac{P_T}{d(\mathbb{X})}\Big)^{1-\kappa}\nonumber\\
&\quad+\frac{G^2\alpha_0}{2(1-c)}T^{1-c}\Big(1+\frac{P_T}{d(\mathbb{X})}\Big)^\kappa.
\end{align}

Applying Lemma~\ref{online_op:lemma_expert} to \eqref{online_op:al_w} and \eqref{online_op:al_x} yields
\begin{align*}
&\sum_{t=1}^{T}\ell_t(x_t)-\min_{i\in[N]}\Big\{\sum_{t=1}^{T}\ell_t(x_{i,t})
+\frac{1}{\beta}\ln\frac{1}{w_{i,1}}\Big\}\nonumber\\
&\le\frac{\beta (Gd(\mathbb{X}))^2T}{2},
\end{align*}
which implies
\begin{align}\label{online_op:corollaryregequ1_alg1_proof2}
&\sum_{t=1}^{T}\ell_t(x_t)-\sum_{t=1}^{T}\ell_t(x_{i_0,t})
\le\frac{\beta (Gd(\mathbb{X}))^2T}{2}+\frac{1}{\beta}\ln\frac{1}{w_{i_0,1}}\nonumber\\
&=\frac{\beta_0 (Gd(\mathbb{X}))^2T^{1-c}}{2}+\frac{1}{\beta_0}T^c\ln\frac{1}{w_{i_0,1}}.
\end{align}

From $w_{i,1}=\frac{N+1}{i(i+1)N}$, we have
\begin{align}\label{online_op:corollaryregequ1_alg1_proof3}
\ln\frac{1}{w_{i_0,1}}&\le\ln(i_0(i_0+1))\le2\ln(i_0+1)\nonumber\\
&\le2\ln(\lfloor \kappa\log_2(1+P_T/d(\mathbb{X}))\rfloor).
\end{align}

From \eqref{online_op:ell} and that $f_t(x)$ is convex as assumed in Assumption~\ref{online_op:assfunction}, we have
\begin{align}\label{online_op:corollaryregequ1_alg1_proof4}
\ell_t(x_t)-\ell_t(y_t)=\langle \partial f_t(x_t),x_t-y_t\rangle\ge f_t(x_t)-f_t(y_t).
\end{align}

From \eqref{online_op:corollaryregequ1_alg1_proof1}--\eqref{online_op:corollaryregequ1_alg1_proof4}, we have
\begin{align*}
&\Reg(x_{[T]},y_{[T]})\nonumber\\
&\le \frac{4}{\alpha_0}(d(\mathbb{X}))^2T^c\Big(1+\frac{P_T}{d(\mathbb{X})}\Big)^{1-\kappa}
\nonumber\\
&\quad+\frac{G^2\alpha_0}{2(1-c)}T^{1-c}\Big(1+\frac{P_T}{d(\mathbb{X})}\Big)^\kappa
+\frac{\beta_0 (Gd(\mathbb{X}))^2T^{1-c}}{2}\nonumber\\
&\quad+\frac{2}{\beta_0}T^c\ln(\lfloor \kappa\log_2(1+P_T/d(\mathbb{X}))\rfloor),
\end{align*}
which gives \eqref{online_op:corollaryregequ1_alg1}.

\noindent {\bf (ii)} From \eqref{online_op:stepsize3}, we have
\begin{align}\label{online_op:varepsilon_5}
\varepsilon_{i,4}\le\varepsilon_5,~\forall i\in[N],
\end{align}
where $\varepsilon_5=\frac{4(d(\mathbb{X}))^2}{\alpha_0T^{1-c}}+\frac{2G^2\alpha_0(1+T)^\kappa}{(1-c)T^c}
+\frac{\gamma_0^2F^2}{\alpha_0T}+\frac{2\gamma_0^2F^2}{\alpha_0T^{1-c}}+2Gd(\mathbb{X})$.

From $\|[g(x)]_+\|$ is convex, \eqref{online_op:al_w}, and \eqref{online_op:al_x}, we have
\begin{align}\label{online_op:corollaryregequ1_alg1_proof5}
&\sum_{t=1}^{T} \|[g(x_{t})]_+\|=\sum_{t=1}^{T} \|[g(\sum_{i=1}^{N}w_{i,t}x_{i,t})]_+\|\nonumber\\
&\le\sum_{t=1}^{T}\sum_{i=1}^{N}w_{i,t} \|[g(x_{i,t})]_+\|\le\sum_{t=1}^{T}\sum_{i=1}^{N} \|[g(x_{i,t})]_+\|\nonumber\\
&=\sum_{i=1}^{N}\sum_{t=1}^{T} \|[g(x_{i,t})]_+\|.
\end{align}

From \eqref{online_op:corollaryconsequ_proof2_dr2}, \eqref{online_op:varepsilon_5}, and \eqref{online_op:corollaryregequ1_alg1_proof5}, we have
\begin{align*}
&\sum_{t=1}^{T}\|[g(x_{t})]_+\|\nonumber\\
&\le\sum_{i=1}^{N}\frac{\sqrt{m\alpha_02^{i-1}\varepsilon_5}}{\gamma_0}
\Big(T^{(1-c)/2}+\frac{1}{1-c}T^{(1-c)/2}\Big)\nonumber\\
&\le\frac{\sqrt{m\alpha_02^{N}\varepsilon_5}}{\gamma_0(\sqrt{2}-1)}
\Big(T^{(1-c)/2}+\frac{1}{1-c}T^{(1-c)/2}\Big)\nonumber\\
&
\le\frac{2\sqrt{m\alpha_0\varepsilon_5}}{\gamma_0(\sqrt{2}-1)}(T+1)^{\kappa/2}
\Big(T^{(1-c)/2}+\frac{1}{1-c}T^{(1-c)/2}\Big),
\end{align*}
which yields \eqref{online_op:corollaryconsequ_alg1}.

\end{document}